\documentclass[english]{article}
\usepackage[final, nonatbib]{neurips_2020}

\usepackage[utf8]{inputenc} %
\usepackage[T1]{fontenc}    %
\usepackage{hyperref}       %
\usepackage{url}            %
\usepackage{booktabs}       %
\usepackage{amsfonts}       %
\usepackage{nicefrac}       %
\usepackage{microtype}      %
\usepackage[inline]{enumitem}

\usepackage{mathptmx}    %
\usepackage[T1]{fontenc}
\usepackage{babel}
\usepackage{amsmath}
\usepackage{amssymb}
\usepackage{amsthm}
\usepackage{graphicx}
\usepackage{subcaption}
\usepackage{siunitx}
\usepackage{breakurl}
\usepackage{todonotes}
\usepackage{verbatim}
\usepackage{mathtools}
\mathtoolsset{showonlyrefs=true}

\usepackage{wrapfig}
\usepackage{algorithm}
\usepackage{algpseudocode}
\usepackage{float}
\usepackage{siunitx}
\usepackage[inline]{enumitem}
\usepackage{thm-restate}

\newtheorem{theorem}{Theorem}

\newtheorem{definition}{Definition}

\newtheorem{remark}{Remark}

\newtheorem{lemma}{Lemma}
\newtheorem{corollary}{Corollary}

\newcommand{\EE}{\mathbb{E}}
\newcommand{\R}{\mathbb{R}}

\newcommand{\Prob}{\mathbb{P}}

\newcommand{\bx}{\mathbf{x}}
\newcommand{\bw}{\mathbf{w}}
\newcommand{\by}{\mathbf{y}}
\newcommand{\bY}{\mathbf{Y}}
\newcommand{\bX}{\mathbf{X}}
\DeclareMathOperator{\argmax}{arg\max}
\DeclareMathOperator{\argmin}{arg\min}

\newcommand{\cX}{\mathcal{X}}

\newcommand{\sign}{\text{sign }}

\newcommand{\supp}[1]{\operatorname{supp}(#1)}

\title{A Randomized Algorithm to Reduce the Support of Discrete Measures} %

\author{
Francesco Cosentino\\
Mathematical Institute \\
University of Oxford \\
The Alan Turing Institute\\
\small{francesco.cosentino@maths.ox.ac.uk}
\And
Harald Oberhauser\\
Mathematical Institute \\
University of Oxford \\
\small{oberhauser@maths.ox.ac.uk}
\And
Alessandro Abate\\
Dept. of Computer Science \\
University of Oxford \\
The Alan Turing Institute\\
\small{alessandro.abate@cs.ox.ac.uk}
}

\begin{document}
\maketitle
\begin{abstract}
Given a discrete probability measure supported on $N$ atoms and a set of $n$ real-valued functions, there exists a probability measure that is supported on a subset of $n+1$ of the original $N$ atoms and has the same mean when integrated against each of the $n$ functions.
If $ N \gg n$ this results in a huge reduction of complexity.
We give a simple geometric characterization of barycenters via negative cones and derive a randomized algorithm that computes this new measure by ``greedy geometric sampling''.
We then study its properties, and benchmark it on synthetic and real-world data to show that it can be very beneficial in the $N\gg n$ regime.
A Python implementation is available at \url{https://github.com/FraCose/Recombination_Random_Algos}.
\end{abstract}

\section{Introduction}\label{sec:intro}
Discrete probability measures are central to many inference tasks, for example as empirical measures.
In the ``big data'' regime, where the number $N$ of samples is huge, this often requires to construct a reduced summary of the original measure.
Often this summary is constructed by sampling $n$ points at random out of the $N$ points, but Tchakaloff's theorem suggests that there is another way.  
\begin{theorem}[Tchakaloff~\cite{Tchak}]\label{th:tchakalof}
  Let $\mu$ be a discrete probability measure that is supported on $N$ points in a space $\cX$.
  Let $\{f_1,\ldots,f_n\}$ be a set of $n$ real-valued functions $f_i:\cX \rightarrow \R$, $n<N$.
  There exists a discrete probability measure $\hat \mu$ such that $\supp{\hat \mu} \subset \supp{\mu}$, $|\supp{\hat \mu}|\le n+1$, and 
  \begin{align}\label{eq:F}
  \EE_{X \sim \mu}[f_i(X)] = \EE_{X\sim \hat\mu}[f_i(X)] \text{ for all } i\in \{1,\ldots,n\}.
  \end{align}
\end{theorem}
We introduce a randomized algorithm that computes $\hat \mu$ efficiently in the $n \ll N$ regime. 

\paragraph{Related work.}
Reducing the support of a (not necessarily discrete) measure subject to matching the mean on a set of
functions is a classical problem, which goes back at least to Gauss' famous quadrature formula that matches the mean of monomials up to a given degree when $\mu$ is the Lebesgue measure on $\R$.
In multi-dimensions this is known as cubature and Tchakaloff \cite{Tchak} showed the existence of a reduced measure for compactly supported, not necessarily discrete, measures, see \cite{Bayer2006}. 
When $\mu$ is discrete, the problem of computing $\hat \mu$ also runs under the name recombination.
Algorithms to compute $\hat \mu$ for discrete measures $\mu$ go back at least to \cite{Davis1967} and have been an intensive research topics ever since; we refer to \cite{Litterer2012} for an overview of the different approaches, and \cite{Litterer2012,maria2016a,Maalouf2019} for recent, state of the art algorithms and applications. 
Using randomness for Tchakaloff's Theorem has been suggested before \cite{piazzon2017caratheodory,hayakawa2020monte} but the focus there is to show that the barycenter lies in the convex hull when enough points from a continuous measure are sampled so that subsequently any of the algorithms \cite{Litterer2012,maria2016a,Maalouf2019} can be applied;
in contrast, our randomness stems from the reduction algorithm itself. %
More generally, the topic of replacing a large data set by a small, carefully
weighted subset is a vast field that has attracted many different communities
and we mention, pars pro toto, computational geometry \cite{agarwal2005geometric}, coresets in computer science \cite{phillips2016coresets}, scalable Bayesian statistics \cite{Huggins2016CoresetsFS}, clustering and optimisation \cite{Feldman2013TurningBD}. 
In follow up work \cite{Cosentino2020a}, our randomized algorithm was already used to derive a novel approach to stochastic gradient descent. 

\paragraph{Contribution.}
The above mentioned algorithms \cite{Litterer2012,maria2016a,Maalouf2019} use a divide and conquer approach that splits up points into groups, computes a barycenter for each group, and solves a constrained linear system several times. 
This leads to a deterministic complexity that is determined by $N$ and $n$. 
In contrast, our approach uses the geometry of cones to ``greedy'' sample for candidates in the support of $\mu$ that are atoms for $\hat \mu$ and tries to construct the reduced measure in one go.
Further, it can be optimized with classical black box reset strategies to reduce the variance of the run time.
Our results show that this can be very efficient in the big data regime $N \gg n$ that is common in machine learning applications, such as least least square solvers when the number of samples $N$ is very large.
Moreover, our approach is complementary to previous work since it can be combined with it: by limiting the iterations for our randomized algorithm and subsequently running any of the deterministic algorithms above if a solution by ``greedy geometric sampling'' was not found, 
one gets a hybrid algorithm that is of the same order as the deterministic one but that has a good chance of being faster; we give full details in Appendix~\ref{sec:combine} but focus in main text on the properties of the randomized algorithm.

\paragraph{Outline.}
We introduce the basic ideas in Section~\ref{sec:naive algo}, where we derive a simple version of the greedy sampling algorithm, and study its theoretical properties. 
In Section~\ref{sec:optimized algo} we optimize the algorithm to better use the cone geometry, combine with reset strategies to reduce the running time variance, and use the Woodbury formula to obtain a robustness result.
In Section~\ref{sec:exp} we discuss numerical experiments that study the properties of the algorithms on two problems:
\begin{enumerate*}[label=(\roman*)]
\item reducing the support of empirical measures; and 
\item least square solvers for large samples. 
\end{enumerate*}
In the Appendix we provide detailed proofs and more background on discrete geometry.

\section{Negative cones and a naive algorithm}\label{sec:naive algo}
\paragraph{Background.}
As is well-known, Theorem~\ref{th:tchakalof} follows from Caratheodory's convex hull theorem
\begin{theorem}[Caratheodory]\label{th:cath}
  Given a set of $N$ points in $\R^n$ and a point $x$ that lies in the convex hull of these $N$ points.
  Then $x$ is a linear combination of at most $n+1$ points from the $N$ points. 
\end{theorem}
It is instructive to recall how Theorem~\ref{th:cath} implies Theorem~\ref{th:tchakalof}.
Therefore define a $\R^n$-valued random variable $F: \Omega=\cX \rightarrow \R^n$ as $F(\omega):=(f_1(\omega),\ldots,f_n(\omega))$ and note that Equation \eqref{eq:F} is equivalent to 
  \begin{align}
  \int_\Omega F(\omega) \mu(d\omega) =\int_\Omega F(\omega) \hat \mu(d\omega).
  \end{align}
  Since $\mu$ has finite support, the left-hand side can be written as a sum $\sum_{\omega \in \supp{\mu}}F(\omega)\mu(\omega)$.
  This sum gives a point in the convex hull of the set of $N$ (or less) points $\bx:=\{F(\omega): \omega \in \supp{\mu}\}$ in $\R^n$.
  But by Caratheodory's theorem, this point must be a convex combination of a subset $\hat \bx$ of only $n+1$ (or less) points of $\bx$ and Theorem~\ref{th:tchakalof} follows.
  This proof of Theorem~\ref{th:tchakalof} is also constructive in the sense that it shows that computing $\hat \mu$ reduces to constructing the linear combination guaranteed by Caratheodory's theorem; e.g.~by solving $N$ times a constrained linear system, see~\cite{Davis1967}.

\paragraph{Barycenters and cones.}
Key to Tchakaloff's theorem is to verify if two measures have the same mean.
We now give a simple geometric characterization in terms of negative cones, Theorem~\ref{th:main}. 
\begin{definition}\label{def:cone}
  Let $\bx \subset \R^n $ be a finite set of points in $\R^n$.
  We call the set 
 $
   C(\bx):=\{c \in \R^n\, |\, c=\sum_{x \in \bx} \lambda_x x, \text{ where }\lambda_x\ge 0\}
  $
  the cone generated by $\bx$ and we also refer to $\bx$ as its basis.
  We call the set
 $
   C^-(\bx):=\{c\, |\, c=\sum_{x \in \bx} \lambda_x x, \text{ where }\lambda_x\le 0\}
 $
  the negative cone generated by $\bx$. %
\end{definition}
For example, $C({x_1,x_2})$ is the ``infinite'' triangle created by the half-lines $\overline{0x_{1}}$, $\overline{0x_{2}}$ with origin in $0$; $C(\{x_{1},x_{2},x_{3}\})$, is the infinite pyramid formed by the edges $\overline{0x_{1}}$, $\overline{0x_{2}}$, $\overline{0x_{3}}$, with vertex 0; see Figure \ref{fig:cones}. 
\begin{figure}[t!]
    \centering
        \includegraphics[height=4.5cm,width=0.45\textwidth]{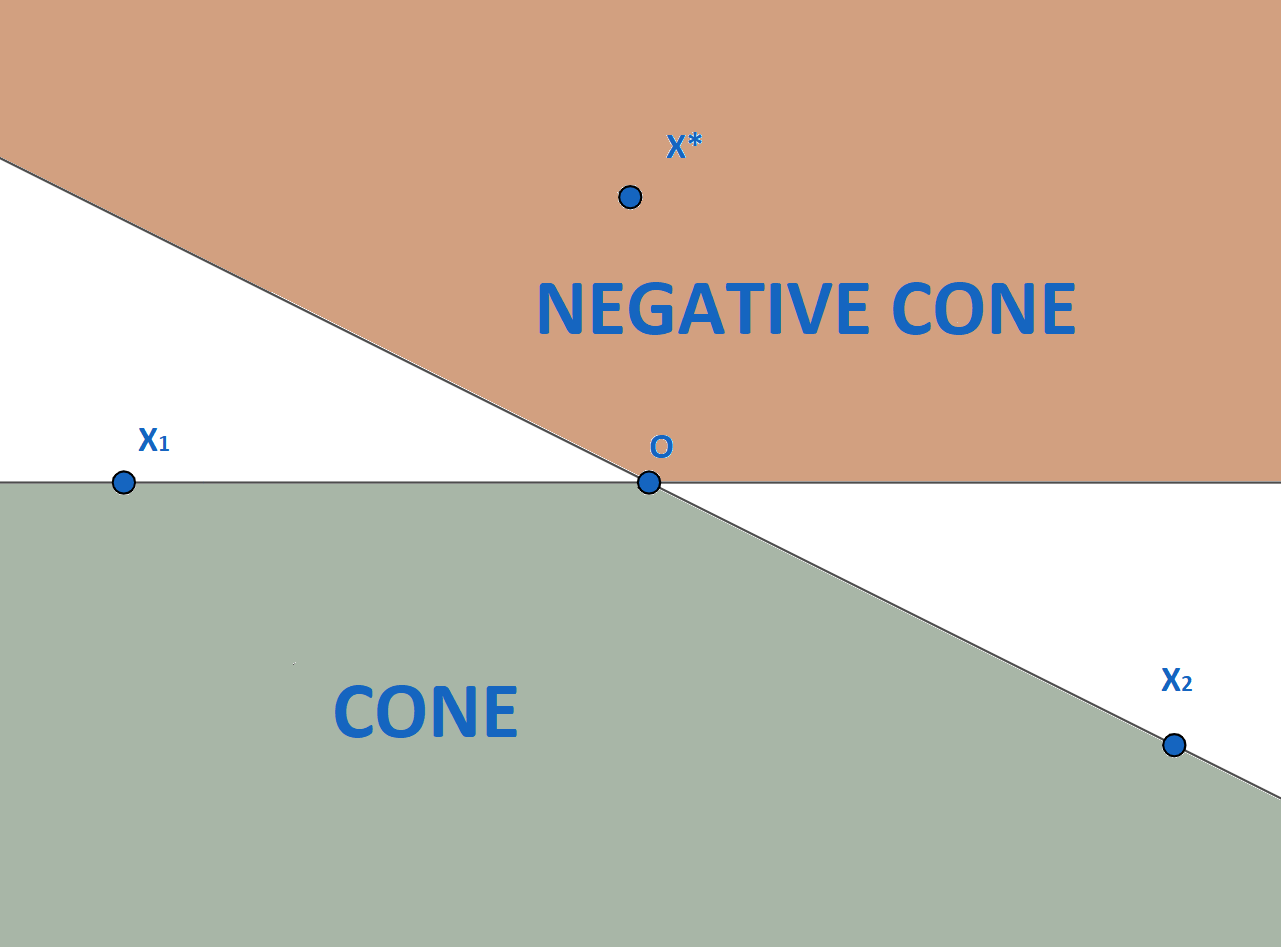}
        \includegraphics[height=4.5cm,width=0.45\textwidth]{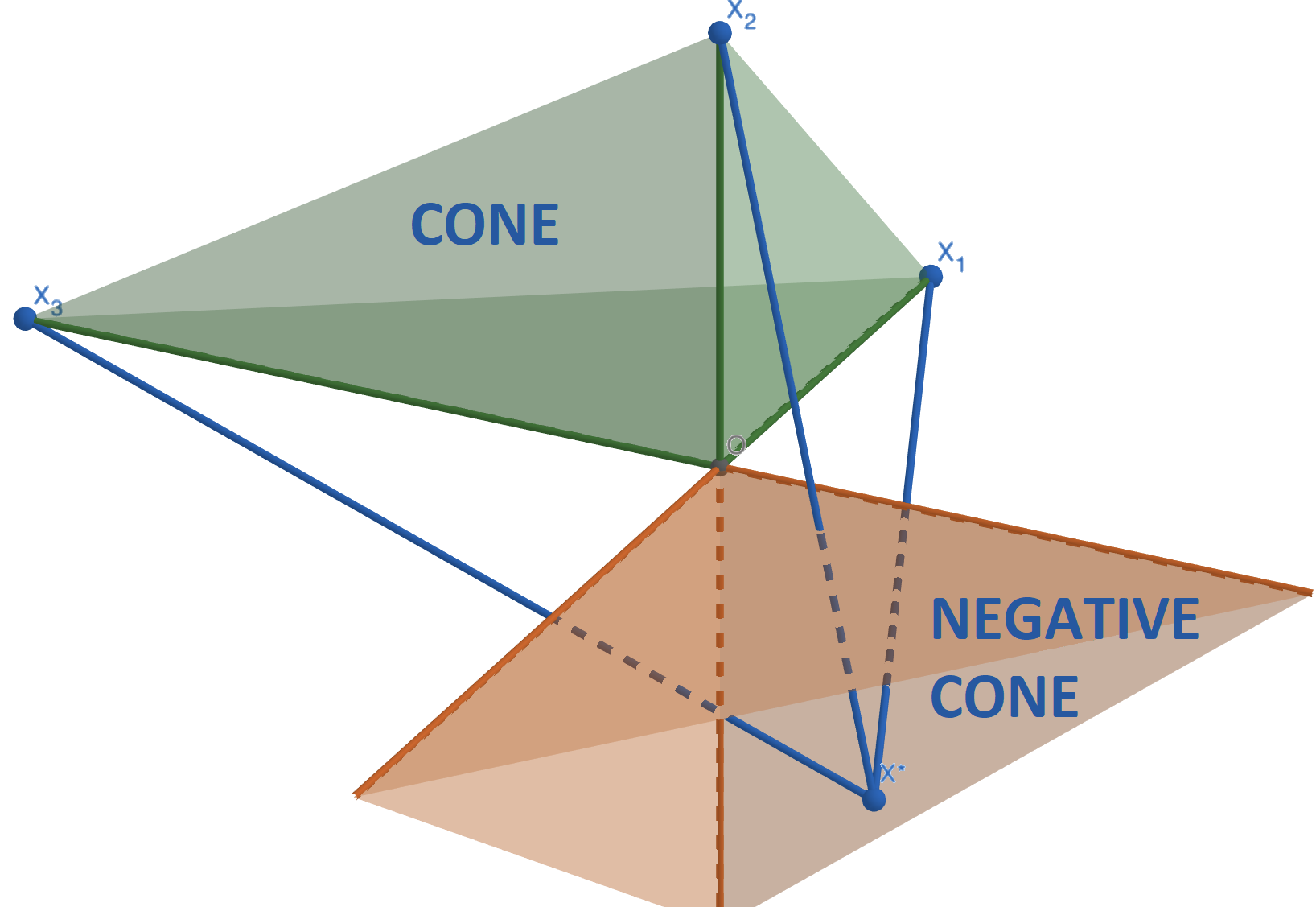}
    \caption{Cones and Negative Cones spanned by two (left) and three points (right).}
    \label{fig:cones}
\end{figure}
\begin{restatable}{theorem}{thmmain}\label{th:main}
Let $\bx=\{x_1,\ldots , x_{n+1}\}$ be a set of $n+1$ points in $\R^n$ such that $ \bx\setminus \{x_{n+1}\}$ spans $\R^n$.
Let $A$ be the matrix that transforms $\bx \setminus \{x_{n+1}\} =\{x_1,\ldots,x_n\}$ to the orthonormal basis $\{e_1,\ldots,e_n\}$ of $\R^n$, i.e. $Ax_i=e_i$.
Further, let $h_i$ be the unit vector such that $\langle h_i, x \rangle =0$ for all $x \in \bx\setminus \{x_i,x_{n+1}\}$ and $\langle h_i, x_i \rangle < 0$ and denote with $H_{\bx\setminus\{x_{n+1}\}}$ a $n\times n$ matrix that has $h_1,\ldots,h_n$ as row vectors.
It holds that
\begin{enumerate}
\item \label{itm:weyl}
  $C(\bx\setminus\{x_{n+1}\})=\{ c| H_{\bx\setminus\{x_{n+1}\}} c \le 0\}$ and $C^-(\bx\setminus\{x_{n+1}\})=\{ c | H_{\bx\setminus\{x_{n+1}\}} c \ge 0\}$.
\item \label{itm: A and H >} 
$Ax \geq 0 $ if and only if $H_{\bx\setminus \{x_{n+1}\}} x \le 0$
and $Ax \leq 0 $ if and only if $H_{\bx\setminus \{x_{n+1}\}} x \geq 0$.
\item \label{itm: barycenter}
There exists a convex combination of $\bx$ with $0$ as barycentre,
  $
   \sum_{i=1}^{n+1} w_i x_i =0 \text{ for some }w_i>0, \text{ and }\sum_{i=1}^{n+1} w_i=1
  $
 if and only if $x_{n+1} \in C^-(\bx \setminus \{x_{n+1}\})$.  
\end{enumerate}
\end{restatable}

The above result could be formulated without the matrix $A$, only in terms of $H_{\bx\setminus \{x_{n+1}\}}$.
However, $A$ is the inverse of the matrix with columns equal to the vectors $\{\bx\setminus \{x_{n+1}\}\}$, hence computing $A$ is more efficient than computing $H_{\bx\setminus \{x_{n+1}\}}$, since matrix inversion is optimized in standard libraries.
\paragraph{A Na\"ive Algorithm.}
Theorem~\ref{th:main} implies a simple randomized algorithm: sample $n$ points at random until the negative cone spanned by the $n$ points is not empty.
Then item~\ref{itm: barycenter} of Theorem~\ref{th:main} implies that the $n$ points in the cone and any point in the negative cone form the support a reduced measure.
If $\EE_\mu[X]\not=0$ we can always study the points $\bx - \EE_\mu[X]$, this is equivalent in the proof of Theorem \ref{th:main} to work with cones whose vertex is not 0.
\begin{algorithm}\caption{Basic measure reduction algorithm}\label{euclid}
  \begin{algorithmic}[1]
    \Procedure{Reduce}{A set $\bx$ of $N$ points in $\R^n$}
    \State Choose $n$ points $\bx^\star$ from $\bx$\label{random start}
     \While{$C^-(\bx^\star)\cap \bx = \emptyset$ }
     \State{Replace $\bx^\star$ with $n$ new random points $\bx^\star$ from $\bx$\label{sample}}
     \EndWhile\label{euclidendwhile}
     \State $\bx^\star \leftarrow \bx^\star \cup x^\star$ with an arbitrary $x ^\star \in C^-(\bx^\star) \cap \bx$
     \State Solve the linear system $\sum_{x \in \bx^\star} w_x^\star x =0$ for
     $w^\star=(w_x^\star)_{x \in \bx^\star}$
     \State \textbf{return}  $(\bx^\star,w^\star)$
    \EndProcedure
  \end{algorithmic}
\end{algorithm}

\begin{corollary}\label{cor:complexity}
  Algorithm~\ref{euclid} computes a reduced measure $\hat \mu $ as required by Theorem~\ref{th:tchakalof} in $\tau \cdot O( n^3 + Nn^2) $ computational steps.
  Here $\tau = \inf \{ i \ge 1: C^-(\bX_i)\cap \bx \not=\emptyset\}$, where $\bX_1, \bX_2,\ldots $ are random sets of $n$ points sampled uniformly at random from $\bx$.
\end{corollary}

The complexity of Algorithm~\ref{euclid} in a single loop iteration is dominated by
\begin{enumerate*}[label=(\roman*)]
\item 
 computing the matrix $A$ that defines the cones $C^-(\bx^\star)$ and $C(\bx^\star)$,
\item checking if there are points inside the cones,
  \item solving a linear system to compute the weights $w_i^\star$.
\end{enumerate*}
Respectively, the worst case complexities are ${O}(n^3)$,
${O}(Nn^2)$ and ${O}(n^3)$, since to check if there are points inside the cones we have to multiply $A$ and $\bX$, where $\bX$ is the matrix whose rows are the vector in $\bx$.
\begin{restatable}{proposition}{propworst}\label{prop:worst}
  Let $N>n+1$ and $\mu$ be a discrete probability with finite support and $f_1,\ldots,f_n$ be as in Theorem~\ref{th:tchakalof}.
  Moreover wlog assume $\EE_{X \sim \mu}[f_i(X)]=0$ for $i=1,\ldots,n$.
  With $p:=\frac{n\cdot n!(N-n)!}{N!}$ it holds that 
$
\EE[\tau] \le \frac{1}{p} \text{ and } Var(\tau) \le \frac{1-p}{p^2}
$
and, for fixed $n$, $\lim_{N\rightarrow \infty}\EE[\tau] =1 $.
\end{restatable}
Not surprisingly, the worst case bound for $\EE[\tau]$ are not practical, and it is easy to come up with examples where this $\tau$ will be very large with high probability, e.g.~a cluster of points and one point far apart from this cluster would result in wasteful oversampling of points from the cluster. 
Such examples, suggest that one should be able to do better by taking the geometry of points into account which we will do in Section \ref{sec:optimized algo}.
However, before that, it is instructive to better understand the properties of Algorithm~\ref{euclid} when applied to empirical measures. 
\paragraph{Application to empirical measures.}
Consider a random probability measure $\mu=\frac{1}{N}\sum_{i=1}^{N}\delta_{\left(f_{1}(X_{i}),\ldots,f_{n}(X_{i})\right)} $ where the $X_1,X_2,\ldots$ are independent and identically distributed. 
\begin{restatable}{proposition}{propcomplexity}\label{prop:empirical}
  Let $N>n+1$ and let $f_1, \ldots, f_n$ be $n$ real-valued functions and $X_1,\ldots,X_N$ be $N$ i.i.d.~copies of a random variable $X$. 
  Set $F(X)=(f_1(X),\ldots,f_n(X))$, assume $\EE[F(X)] =0 $ and denote
\begin{align}\label{eq: 0 in conv}
E:=\{0 \in \operatorname{Conv}\{F(X_i), i \in \{1,\ldots,N\}\} \}.
\end{align}
\begin{enumerate}
\item \label{itm:tau general} $\EE[\tau|E] \le \frac{1}{p}$ and $Var(\tau|E) \le \frac{1-p}{p^2}$,
where 
\begin{align}\label{eq:tau_general_estimates}
p=\max\left\{ \frac{n\cdot n!(N-n)!}{N!},1-\Prob\left(0\not\in\operatorname{Conv}\{F(X_{1}),\ldots,F(X_{n+1})\}\right)^{N-n}\right\},
\end{align} 
\item \label{itm:tau symm} 
If the law of  $F(X)$ is invariant under reflection in the origin, then $\Prob\left(0\not\in\operatorname{Conv}\{F(X_{1}),\ldots,F(X_{n+1})\}\right) = 1-2^{-n}$,
\item \label{itm:N_to_infty}
For fixed $n$, as $N\to\infty$
\begin{align}
\Prob(&\text{for $n$ uniformly at random chosen points $\bx^\star$ from $\bx$, }\exists x\in \bx\,\, s.t. \,\,x\in C^-(\bx^\star))\to 1,
\end{align}
where $\bx=\{F(X_1), F(X_2),\ldots,F(X_N)\}$.
\end{enumerate}
\end{restatable}
Proposition~\ref{prop:empirical} conditions on the event \eqref{eq: 0 in conv} so that the recombination problem is well-posed, but this happens with probability one for large enough $N$, see Theorem \ref{th:convergence_CH} in Appendix and \cite{hayakawa2020monte}.   
Not surprisingly, the worst case bounds of Algorithm~\ref{euclid} can be inconvenient, as equation~\eqref{eq:tau_general_estimates} shows. %
Nevertheless, item \ref{itm:tau symm} of Proposition \ref{prop:empirical} shows an interesting trade-off in computational complexity, since the total cost   
\begin{align}\label{eq:theoretical_complexity}
\EE[\tau]O(n^{3}+Nn^{2})\leq C(n^{3}+Nn^{2})\min\left\{ \frac{N!}{n\cdot n!(N-n)!},\frac{1}{1-\left(1-2^{-n}\right)^{N-n}}\right\}
\end{align}
has a local minimum in $N$, see Figure \ref{fig:theoretical_result} in the Appendix. Section \ref{sec:exp} shows that this is observed in experiments and this minimum also motivates the divide and conquer strategy we use in the next section. 

\section{A geometrically greedy Algorithm}\label{sec:optimized algo}
Algorithm~\ref{euclid} produces candidates for the reduced measure by random sampling and then accepts or rejects them via the characterization given in Theorem~\ref{th:main}.
We now optimize the first part of it, namely the selection of candidates, by exploiting better the geometry of cones. %

\paragraph{Motivation in two dimensions.}
Having chosen $\bx^\star$ points we know that we have found a solution if $C^-(\bx^\star) \cap \bx \neq \emptyset$.
Hence, maximizing the volume of the cone increases the chance that this intersection is not empty. 
Indeed, when $n=2$ it is easy to show the following result.
\begin{theorem}\label{th:n=2}
  Let $\bx$ be a set of $N{\geq 3}$ points in $\R^2$ and $x_1 \in \bx$.
  Define $\bx^\star=(x_1^\star, x_2^\star)$, where
  \begin{align}\label{eq:two-dim}
  x^{\star}_2:=\argmax_{x\in\bx\setminus \{x_1^\star\}}\left|\frac{\langle x_1^\star,x\rangle}{\|x_1^\star\|\|x\|}-1\right|.
  \end{align}
There exists a convex combination $\sum_{x \in \bx} w_x x$ of $\bx$ that equals $0$ if and only if $\bx \cap C^-(\bx^\star) \neq \emptyset$.
\end{theorem}
\begin{figure}[ht!]
 \centering
  \includegraphics[width=0.35\textwidth]{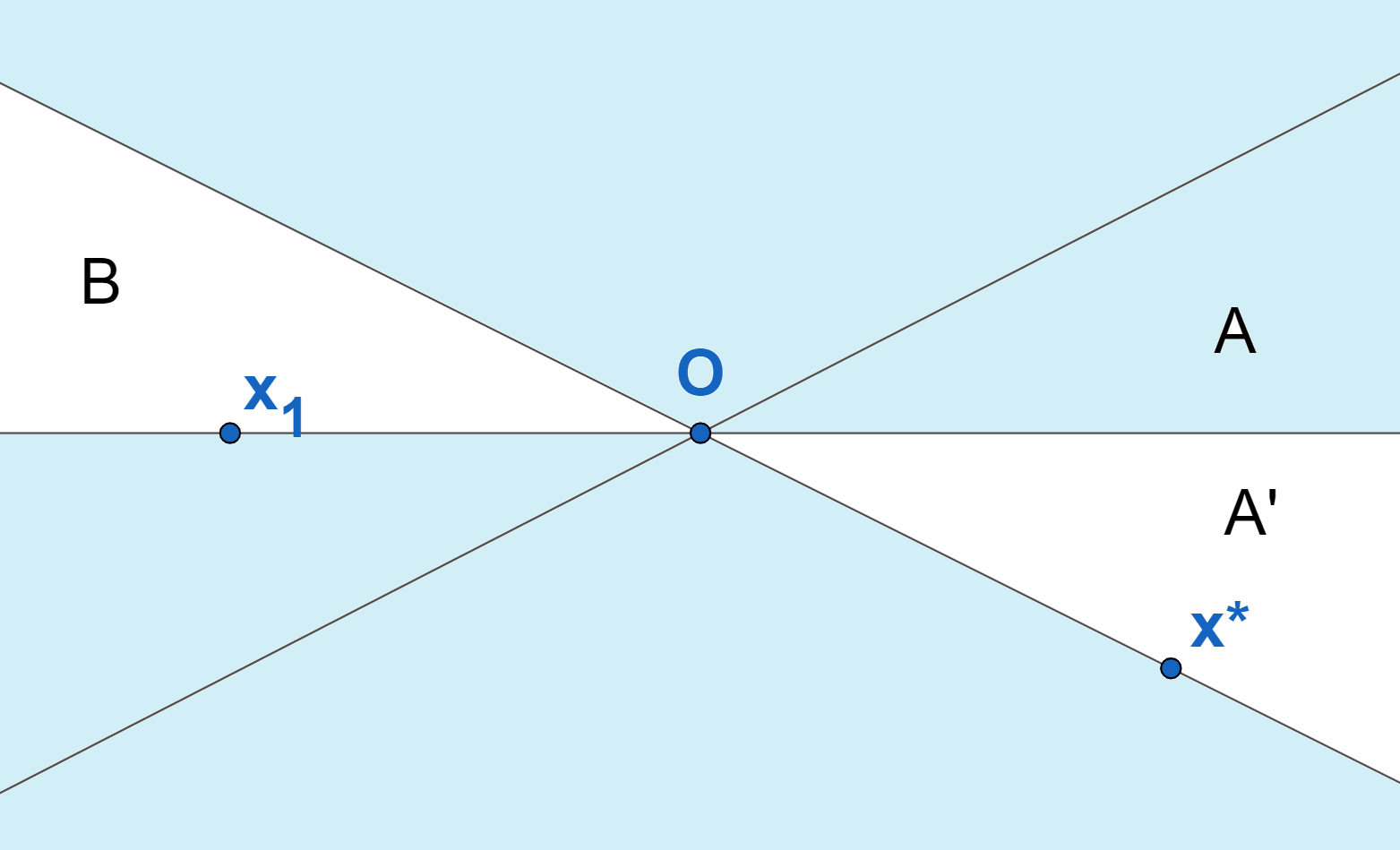}
 \caption{Proof of Theorem~\ref{th:n=2}.}\label{figntwo}
\end{figure} 
Theorem \ref{th:n=2} follows immediately from the cone geometry but it is instructive to spell it out since it motivates the case of general $n$.
\begin{proof}
{
The shaded areas in Figure~\ref{figntwo} indicate $C^-(x_1,x^{\star})$ (on the top) and $C(x_1,x^{\star})$ (on the bottom).
Moreover, by definition of $x^{\star}$ in the region $A$ and $A'$ there are no points.

$(\Rightarrow)$ If there exists a point $x_2$ in $C^-(x_1,x^{\star})$, then by convexity it follows that there exists a convex combination of $0$ for $x_1, x_2,x^{\star}$.\\
$(\Leftarrow)$ If there does not exist a point in $C^-(x_1,x^{\star})$, then there are points only in $B\cup C(x_1,x^{\star})$, therefore, again for simple convex geometry arguments,  it is impossible that there exists a convex combination of $0$ for $x_1, x_2,x^{\star}$.}
\end{proof}
Hence, if we modify Algorithm~\ref{euclid} by selecting in step~\ref{sample} the new point by maximizing the angle according to~\eqref{eq:two-dim} then for $n=2$, Theorem~\ref{th:n=2} guarantees that $n+1=3$ points out of $N$ are found that constitute a reduced measure $\hat \mu$ in $\tau \le 2 $ computational steps.

\paragraph{A geometrically greedy Algorithm.}
For general dimensions $n$, the intuition remains that a good (but not perfect) proxy to maximize the likelihood that the negative cone is non-empty, is given by maximizing the volume of the cone, see \cite[Chapter 8]{Schneider2008}.
Such a volume maximization is a long-standing open question and goes at least back to \cite{Sommerville1927}; see \cite{Bollobas1997} for an overview.
One reviewer, also pointed to recent papers that implicitly apply a similar intuition to other problems, see \cite{Clarkson2010, Aljundi2019}.
All this motivates the ``\emph{geometrically greedy}'' Algorithm~\ref{euclid_opt} that applies for any $n$. 
\begin{algorithm}[tbh]\caption{Optimized measure reduction algorithm}\label{euclid_opt}
\begin{algorithmic}[1]
    \Procedure{Reduce-Optimized}{A set $\bx$ of $N$ points in $\R^n$}
    \State Choose $n$ points $\bx^\star=\{x_1^\star,\ldots,x_n^\star\}$ from $\bx$\label{euclid_opt sample}
    \State {$i\gets 0$}
     \While{$C^-(\bx^\star)\cap \bx = \emptyset$ }
     \State{$\bx \gets \bx \setminus \text{interior}\{C(\bx^\star)\} $  }\label{step:cancel_points}
  \If{$i=0$} \label{step:if}\State{$x_{i+1}^{\star}\gets \argmax_{x\in\bx\setminus\bx^{\star}}\big|\langle x,\sum_{j=2}^{n}x_{j}^{\star}\rangle-1\big|$}
  \Else \State{$x_{i+1}^\star \gets \argmax_{x\in\bx\setminus\bx^{\star}}\big|\langle x,\sum_{j=1}^{i}x_{j}^{\star}\rangle-1\big|$}
  \EndIf
     \State {$i\gets ((i+1) \mod n)$}
     \EndWhile\label{step:euclidendwhile_opt}
     \State $\bx^\star \gets \bx^\star \cup x^\star$ with an arbitrary $x ^\star \in C^-(\bx^\star) \cap \bx$
     \State Solve the linear system $\sum_{x \in \bx^\star} w_x^\star x =0$ for $w^\star=(w_x^\star)_{x \in \bx^\star}$
     \State \textbf{return}  $(\bx^\star,w^\star)$ 
    \EndProcedure
  \end{algorithmic}
\end{algorithm}
First note that the deletion of points in step \ref{step:cancel_points} in Algorithm \ref{euclid_opt} does not throw away potential solutions: suppose we have deleted a point $\hat x$ in the previous step that was part of a solution, i.e.~there exists a set of $n+1$ points $\hat\bx^\star$ in $\bx$ such that $\hat x\in \hat\bx^\star$, and there exist $n+1$ weights $\hat w_i$, $i\in[0,\ldots,n+1]$, $\hat w_i\in[0,1]$, $\sum w_i =1$ such that $\sum_{i=1}^{n+1}\hat{w}_{i}\hat{x}_{i}^{*}=0$, $\hat{x}_{i}^{*}\in\hat{\bx}^{\star}$.
If we indicate with $\bx_c$ the $n$ vectors of the basis of the cone of the previous iteration, we know that $\hat x \in \text{interior}\{C(\bx_c)\}$, which means there exist strictly positive values $c_i$ such that $\hat{x}=\sum_{i=1}^{n}c_{i}x_{i}^{c}$, $x_{i}^{c}\in\bx^{c}$. Therefore,
\begin{align}
\sum_{i=1}^{n+1}\hat{w}_{i}\hat{x}_{i}^{*}=&\hat{w}_{n+1}\hat{x}+\sum_{i=1}^{n}\hat{w}_{i}\hat{x}_{i}^{*}=\hat{w}_{n+1}\sum_{i=1}^{n}c_{i}x_{i}^{c}+\sum_{i=1}^{n}\hat{w}_{i}\hat{x}_{i}^{*},\,\,\,x_{i}^{c}\in\bx^{c}\text{ and }\hat{x}_{i}^{*}\in\hat{\bx}^{\star}.
\end{align}
Given that $w_i$ and $c_i$ are positive $0\in\operatorname{Conv}\{\hat\bx^\star\cup\bx_c\}$ and we can apply again the Caratheodory's Theorem \ref{th:cath}, which implies that the deleted point $\hat x$ was not essential.
The reason for the if clause in step \ref{step:if} is simply that the first time the loop is entered we optimize using the randomly selected bases, but in subsequent runs it is intuitive that we should only optimize over the base points that were optimized in previous loops.
\paragraph{Complexity.}
We now discuss the complexity of Algorithm~\ref{euclid_opt}. 
\begin{restatable}{proposition}{propcomplexityopt}\label{prop:compl_algo_opt}
The complexity of Algorithm~\ref{euclid_opt} to compute a reduced measure $\hat \mu $ ,as in Theorem~\ref{th:tchakalof}, is 
\begin{align}
O(n^{3}+n^{2}N)+(\tau-1)O(n^{2}+nN),
\end{align}
here $\tau = \inf \{ i \ge 1: C^-(\bX_i)\cap \bx \not=\emptyset\}$ where $\bX_1, \bX_2,\ldots $ are obtained as in Algorithm~\ref{euclid_opt}.
\end{restatable}
In contrast, to the complexity of Algorithm~\ref{euclid}, Corollary~\ref{cor:complexity}, the $n^3$ term that results from a matrix inversion is no longer proportional to $\tau$, and the random runtime $\tau$ only affects the complexity proportional to $n^2+nN$. 
For a generalization of Theorem~\ref{th:n=2} from $n=2$ to general $n$, that is a statement of the form ``in $n$ dimensions the algorithm terminates after at most $\tau\le f(n)$'', one ultimately runs into the result of  a ``positive basis'' from discrete geometry, see for example \cite{Regis2016}, that says if $n\geq 3$ it is possible to build positive independent set of vectors of any cardinality.
Characterizing the probability of the occurrences of such sets is an ongoing discrete geometry research topic and we have nothing new to contribute to this.
However, despite the existence of such ``positive independent sets'' for $n\geq 3$, the experiments in Section~\ref{sec:exp} underlines the intuition that in the generic case, maximizing the angles is hugely beneficial also in higher dimension.
If a deterministic bound on the runtime is crucial, one can combine the strengths of Algorithm~\ref{euclid_opt} (a good chance of finding $\bx^\star$ quickly by repeatedly smart guessing) with the strength of deterministic algorithms such as \cite{Litterer2012,maria2016a,Maalouf2019} by running Algorithm~\ref{euclid_opt} for at most $k$ steps and if a solution is not found run a deterministic algorithms.
Indeed, our experiments show that this is on average a very good trade-off since most of the probability mass of $\tau$ is concentrated at small $k$, see also Appendix~\ref{sec:combine}.  
\paragraph{Robustness.}
An interesting question is how robust the measure reduction procedure is to the initial points.
Therefore assume we know the solution of the recombination problem (RP) for $\bx \subset \R^n$, i.e.~a subset of $n+1$ points $\hat\bx=(\hat x_1, \hat x_2,\dots,\hat x_{n+1})\subset	\bx$ and a discrete measure $\hat\mu$ on $\hat \bx$ such that $\hat\mu(\hat\bx)=0$. 
If a set of points $\by$ is close to $\bx$ one would expect that one can use the solution of the RP for $\bx$ to solve the RP for $\by$. 
The theorem below uses the Woodbury matrix identity to make this precise. 
\begin{restatable}{proposition}{thmrobust}\label{th:robust}
  Assume that $\operatorname{span}(\hat\bx)=\operatorname{span}(\hat\bx_{-1})=\R^n$, where $\hat\bx_{-i}:=\hat\bx\setminus{\hat x_i}$.
  Denote with $\bX$ a matrix which as has rows the vectors in $\bx$.
  Suppose there exists an invertible matrix $R$ and another matrix $E$, such that $\bX=\bY R+E$.
  Denote $\gamma_1:= (\hat\bX_{-1}^\top)^{-1} \hat{X}^\top_1$, where  $\hat\bx$ is a solution to the RP $\bx$.
  Assuming that the inverse matrices exist, $\hat\bX R+E_{\hat\bx}$ is a solution to the RP $\by$ if and only if 
\begin{align}\label{eq:robust}
\gamma_{1}^{\top}+E_{\hat{x}_{1}}R^{-1}A_{1}^{\top}\leq&\left(\gamma_{1}^{\top}+E_{\hat{x}_{1}}R^{-1}A_{1}^{\top}\right)E_{\hat{\bx}_{-1}}\left(I+R^{-1}A_{1}^{\top}E_{\hat{\bx}_{-1}}\right)R^{-1}A_{1}^{T}
\end{align}
where $E_y$ indicates the part of the matrix $E$ related to the set of vectors $y \subset \bx$ and $A_1 = (\hat\bX_{-1}^\top)^{-1}$.
\end{restatable}
This is not only of theoretical interest, since the initial choice of a cone basis in Algorithm~\ref{euclid_opt} can be decisive. 
For example, if we repeatedly solve a ``similar'' RP, e.g.~$N$ points are sampled from the same distribution,
then Proposition~\ref{th:robust} suggests that after the first RP solution, we should use the points in the new set of points that are closest to the previous solution as initial choice of the cone basis.

\paragraph{Las Vegas resets.}
The only source of randomness in Algorithm~\ref{euclid_opt} is the choice of $\bx^\star$ in step \ref{euclid_opt sample}.
If this happens to be a bad choice, much computational time is wasted, or even worse, the Algorithm might not even terminate.
However, like any other random ``black box'' algorithm one can stop Algorithm~\ref{euclid_opt} if $\tau$ becomes large and then restart with a random basis $\bx^\star$ sampled independently from the previous one.
We call a sequence $\mathcal{S}=(t_1,t_2,\ldots)$ of integers a reset strategy where the $i$th entry $t_i$ denotes the number of iterations we allow after the $i$th time the algorithm was restarted, e.g.~$\mathcal{S}=(10,3,6,\ldots)$ means that if a solution is not found after $10$ loops, we stop and restart; then wait for at most $3$ loop iterations before restarting; then $6$, etc.
Surprisingly, the question which strategy $\mathcal{S}$ minimises the expected run time of Algorithm~\ref{euclid_opt} has an explicit answer.  
A direct consquence of the main result in \cite{Luby1993} is that 
$
  \mathcal{S}^\star=c\times(1,1,2,1,1,2,4,1,1,2,1,1,2,4,8,1\dots)
$
achieves the best expected running time up to a logarithmic factor, i.e.~by following $\mathcal{S}^\star$ the expected running time is $O(\EE[\tau^\star] \log \EE[\tau^\star])$ where $\tau^\star$ denotes the minimal expected run time under \emph{any} reset strategy.
Thus, although in general we do not know the optimal reset strategy $\tau^\star$ (which will be highly dependent on the points $\bx$), we can still follow $\mathcal{S}^\star$
which guarantees that Algorithm~\ref{euclid_opt} terminates and that its expected running time is within a logarithmic factor of the best reset strategy. 
Since Algorithm~\ref{euclid_opt} uses $n$ updates of a cone basis, it is natural to take $c$ proportional to $n$ (in our experiments we fixed throughout $c=2n$). 
\paragraph{Divide and conquer.}\label{sec:hier_clust}
A strategy used in existing algorithms \cite{Litterer2012,maria2016a,Maalouf2019} is for a given $\mu=\sum_{i=1}^N w_ix_i$ to partition the $N$ points into $2(n+1)$ groups $I_1,\ldots,I_{2(n+1)}$ of approximately equal size, compute the barycenter $b_i$ for each group $I_i$, and then carry out any measure reduction algorithm to reduce the support of $\sum_{i=1}^{2(n+1)} \frac{w_i}{\sum_{j \in I_i} w_j}\delta_{b_i}$ to $n+1$ points.
One can then repeat this procedure, see \cite[Algorithm 2 on p1306]{Litterer2012} for details, and since each iteration halves the support, this terminates after $\log(N/n)$ iterations. %
Hence the total cost is 
${O}(Nn+\log(N/n)) C(2(n+1),n+1)),$ 
where $C(2(n+1),n+1)$ denotes the cost to reduce the measure from $2(n+1)$ points to $n+1$ points.
For the algorithm in \cite{Litterer2012} $C(2(n+1),n+1)={O}(n^4)$, similarly to \cite{Maalouf2019}; for the algorithm in \cite{maria2016a} $C(2(n+1),n+1)={O}(n^3)$.
Similarly, we can also run our randomized Algorithm~\ref{euclid_opt} on the $2(n+1)$ points.
However, the situation is more subtle since the regime where Algorithm~\ref{euclid_opt} excels is the big data regime $N \gg n$, so running the algorithm on smaller groups of points could reduce the effectiveness. 
Already for simple distributions and Algorithm~\ref{euclid}, as in Proposition \ref{prop:empirical} item \ref{itm:tau symm}, we see that for the optimal choice we should divide the points in $N^*_n$ groups with $N^*_n$ denoting the argmin of the complexity in $N$.
This decreases the computational cost to ${O}(Nn+\log_{N^*_n/n} (N) {\bar C}(N^*_n,n+1)),$ where ${\bar C}$ denotes the computational cost of Algorithm \ref{euclid} to reduce from $N^*_n$ points to $(n+1)$ points.
In general, the optimal $N^*_n$ will depend on the distribution of the points, but an informal calculation shows that $N^*_n = 50(n+1)$ achieves the best trade-off;
see Appendix \ref{sec:choice_div&conq} for details.%

\section{Experiments}\label{sec:exp}
We give two sets of experiments to demonstrate the properties of Algorithm \ref{euclid} and Algorithm \ref{euclid_opt}: \begin{enumerate*}[label=(\roman*)]
\item using synthetic data allows to study the behaviour in various regimes of $N$ and $n$, 
\item on real-world data, following \cite{Maalouf2019}, for fast least square solvers. 
\end{enumerate*}
As baselines we consider two recent algorithms \cite{maria2016a} ({\textit{det3}}) and \cite{Litterer2012} resp.~\cite{Maalouf2019} ({\textit{det4}})\footnote{Although the derivation is different, the resulting algorithms in \cite{Litterer2012, Maalouf2019} are essentially identical; we use the implementation of \cite{Maalouf2019} since do the same least mean square experiment as in \cite{Maalouf2019}.
All experiments have been run on a MacBook Pro, CPU: i7-7920HQ, RAM: 16 GB, 2133 MHz LPDDR3. 
}.
\paragraph{Reducing empirical measures.}
We sampled $N \in \{{2n},20, 30, 50, \ldots,10^6\}$ points
\begin{enumerate*}[label=(\roman*)]
\item 
a $n=15$-dimensional standard normal random variable (\textit{symmetric1}),
\item a $n=20$-dimensional standard normal random variable (\textit{symmetric2}),
\item $n=20$ dimensional mixture of exponential (\textit{non symmetric}). 
\end{enumerate*}
We then ran Algorithm \ref{euclid} (\textit{basic}), Algorithm \ref{euclid_opt} (\textit{optimized}), as well as Algorithm \ref{euclid_opt} with the optimal Las Vegas reset (\textit{optimized-reset}) and the divide and conquer strategy (\textit{log-opt}); the results are shown in Figure \ref{fig:bVSo}. %
The first row clearly shows that the performance gets best in the big data regime $N \gg n$. 
The biggest effect is how the angle/volume optimization of Algorithm~\ref{euclid_opt} drastically reduces the number of iterations compared to Algorithm~\ref{euclid}, and therefore the running time and the variance.
From a theoretical perspective is interesting that the shape predicted in Proposition \ref{prop:empirical}, for symmetric distributions such as the normal distribution (the two columns on the left) also manifests itself for the non-symmetric mixture (the right column); see also Figure \ref{fig:theoretical_result} in the Appendix.
The Las Vegas reset strategy is only noticeable in the regime when simultaneously $N$ and $n$ are close and relatively large; e.g.~ Figure \ref{fig:bVSo} falls not in this regime which is the reason why the plots are indistinguishable.
Nevertheless, even in regimes such as in Figure \ref{fig:bVSo} the reset strategy is at least on a theoretical level useful since it guarantees the convergence of Algorithm~\ref{euclid_opt} by excluding pathological cases of cycling through a ``sequence'' of cone bases (although we have not witnessed such pathological cases in our experiments). 
\begin{figure}[bth!]
    \centering
        \includegraphics[height=2.8cm,width=0.32\textwidth]{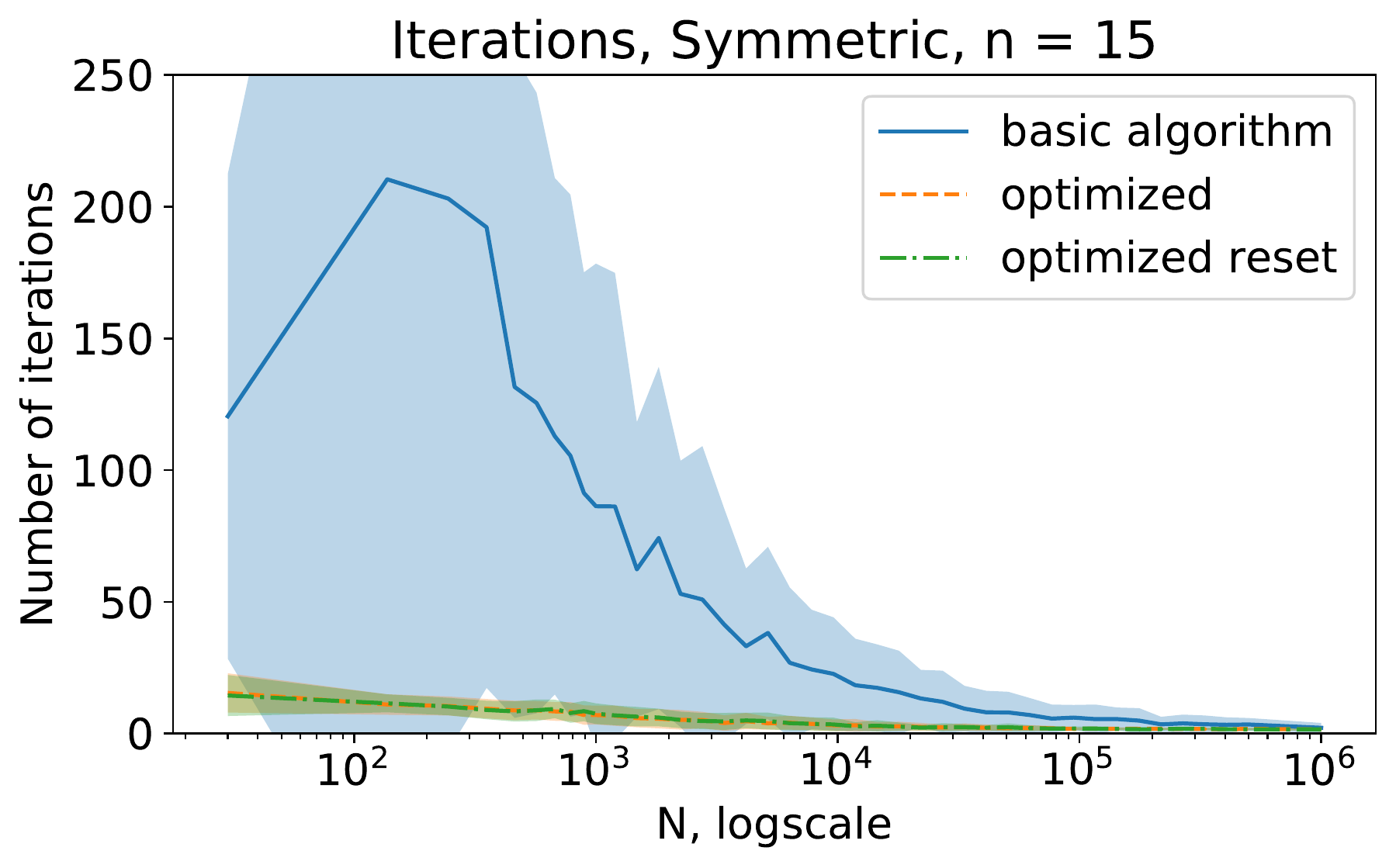}
        \includegraphics[height=2.8cm,width=0.32\textwidth]{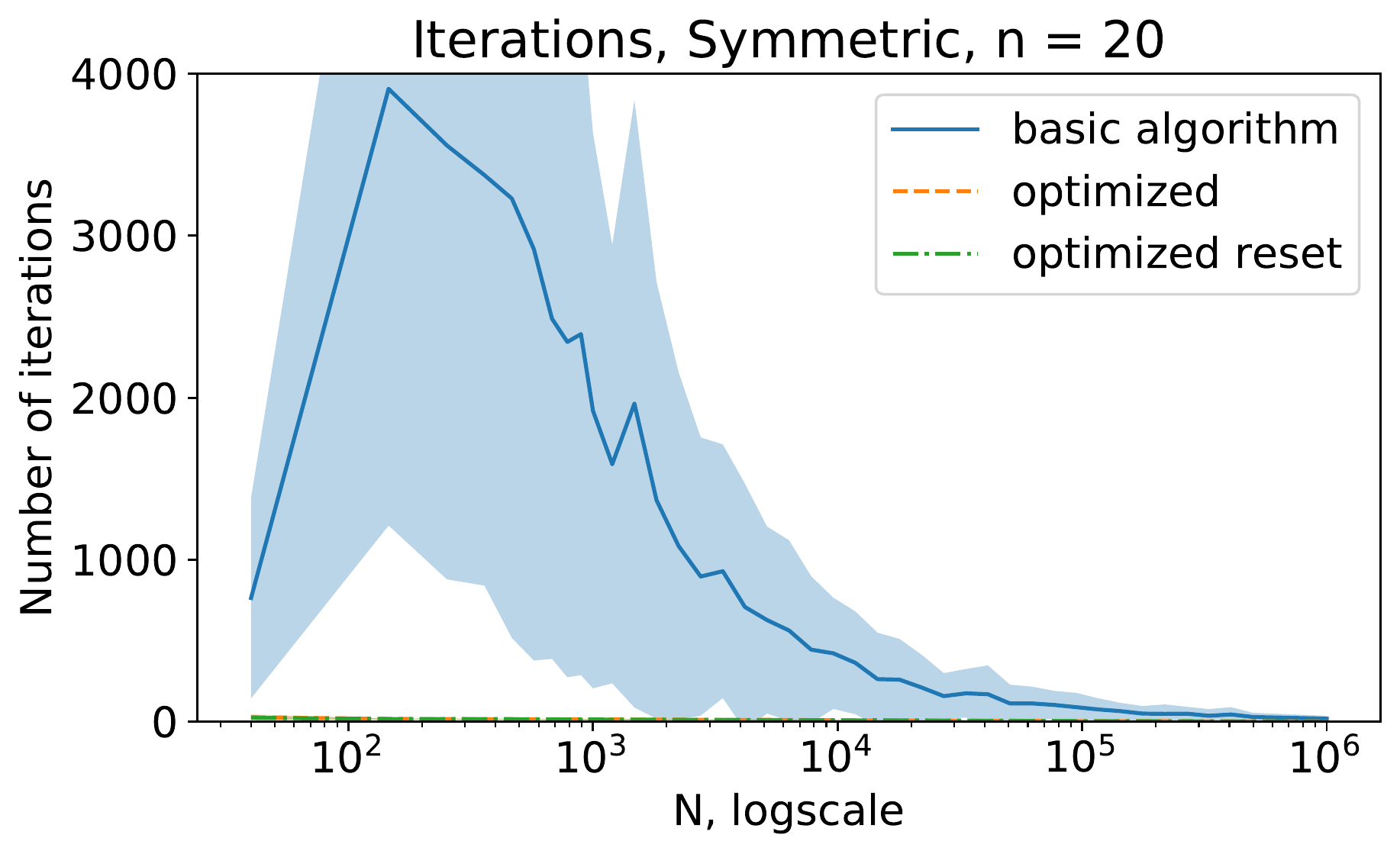}
       \includegraphics[height=2.8cm,width=0.32\textwidth]{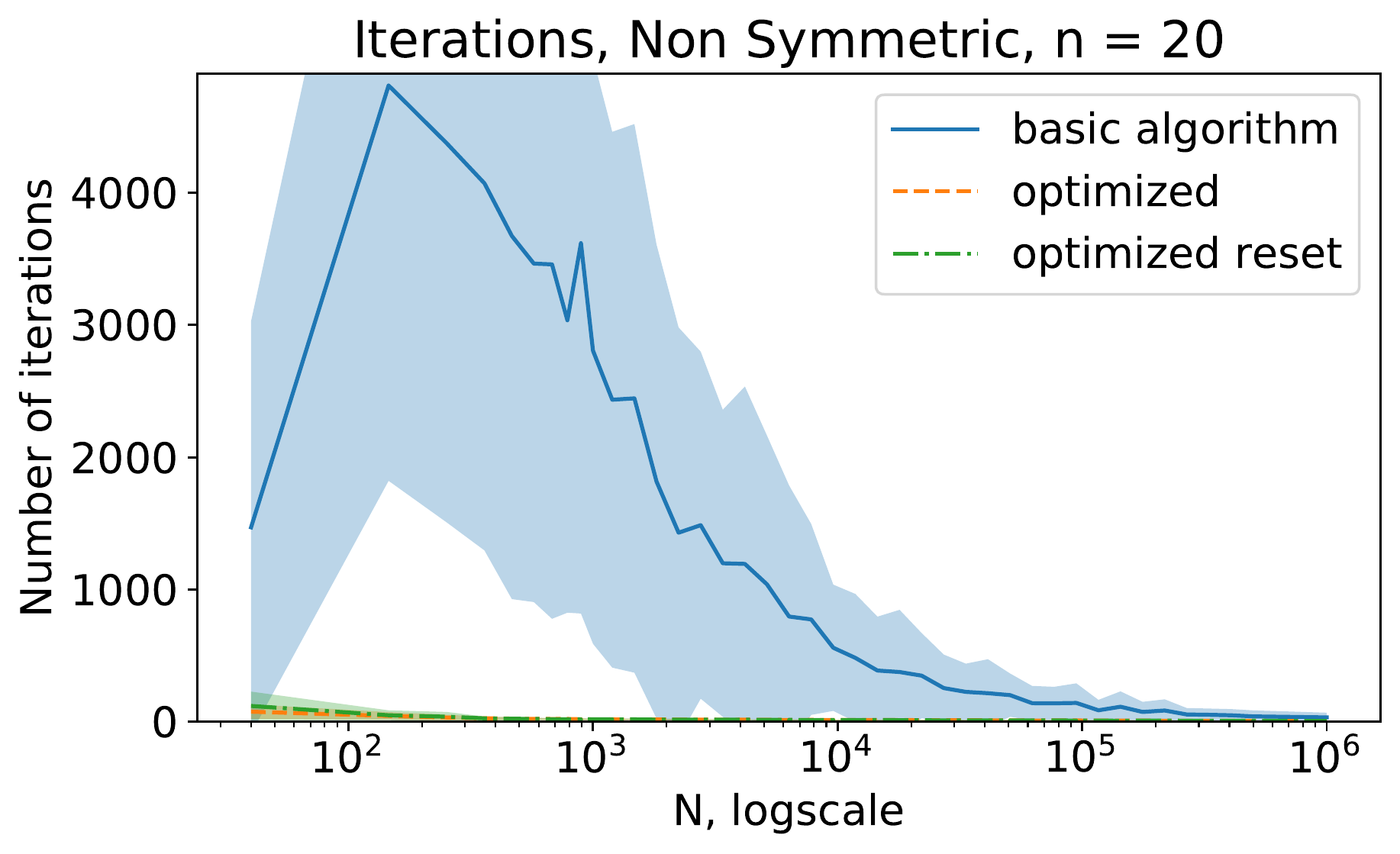}
        \includegraphics[height=2.8cm,width=0.32\textwidth]{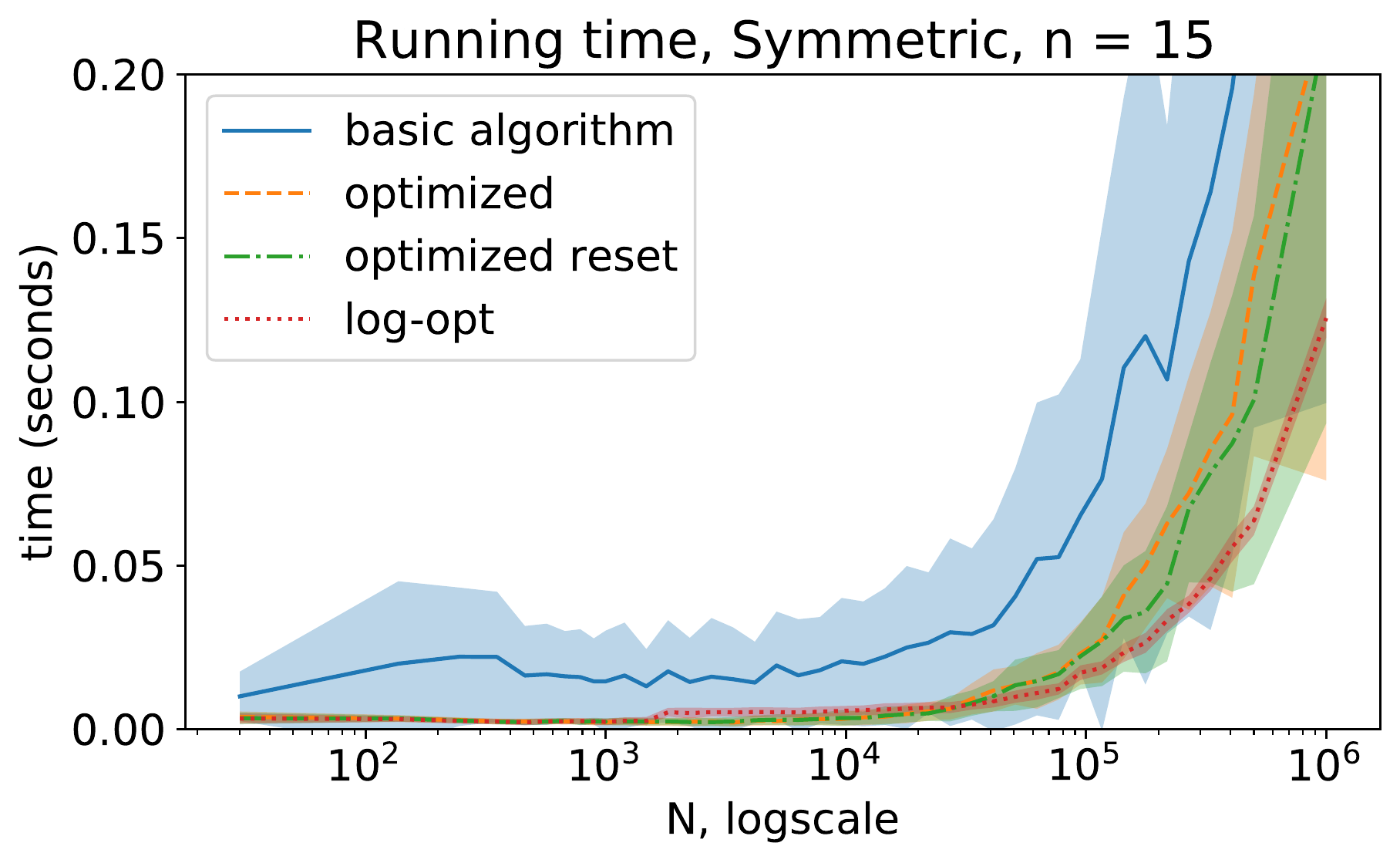}
        \includegraphics[height=2.8cm,width=0.32\textwidth]{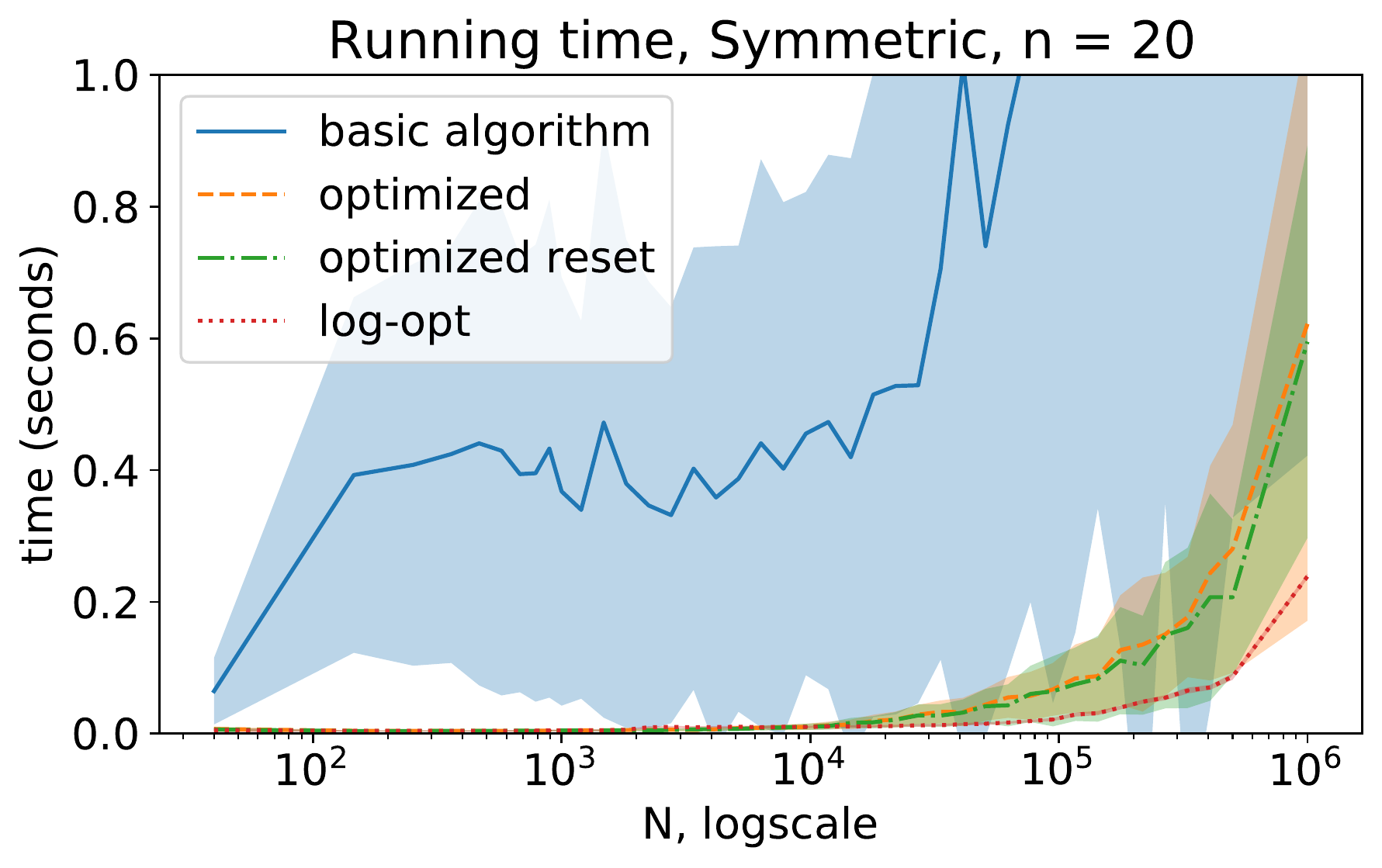}
        \includegraphics[height=2.8cm,width=0.32\textwidth]{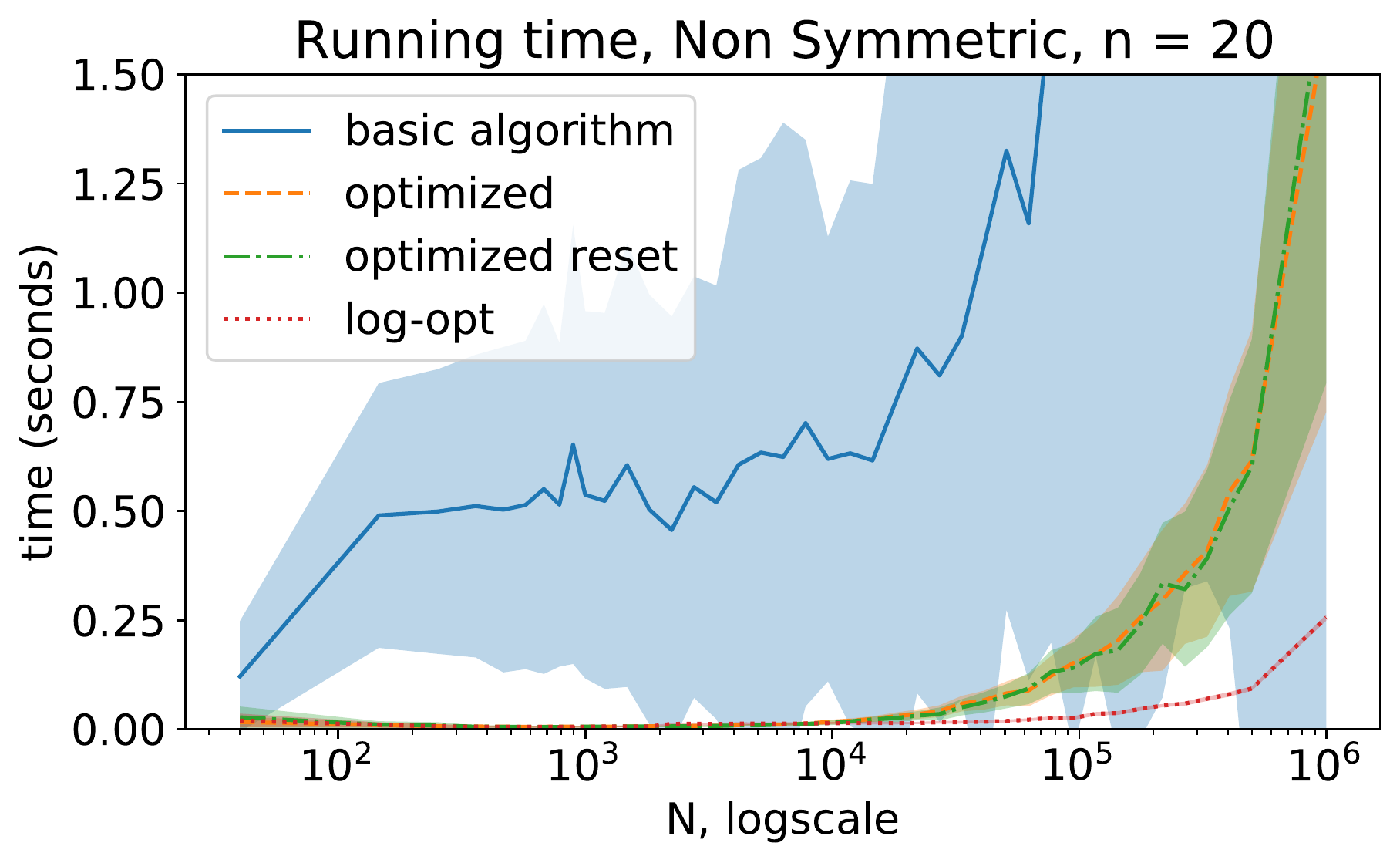}    
        \caption{
Running time and number of iterations of the randomized algorithms as $N$ varies. 
The first two columns show the results for \textit{symmetric1} and \textit{symmetric2} , the right column for \textit{non-symmetric}. 
The shaded area represents the standard deviation (from $70$ repetitions of the experiment).
}
    \label{fig:bVSo}
\end{figure}
\begin{figure}[hbt!]
    \centering
        \includegraphics[height=2.8cm,width=0.32\textwidth]{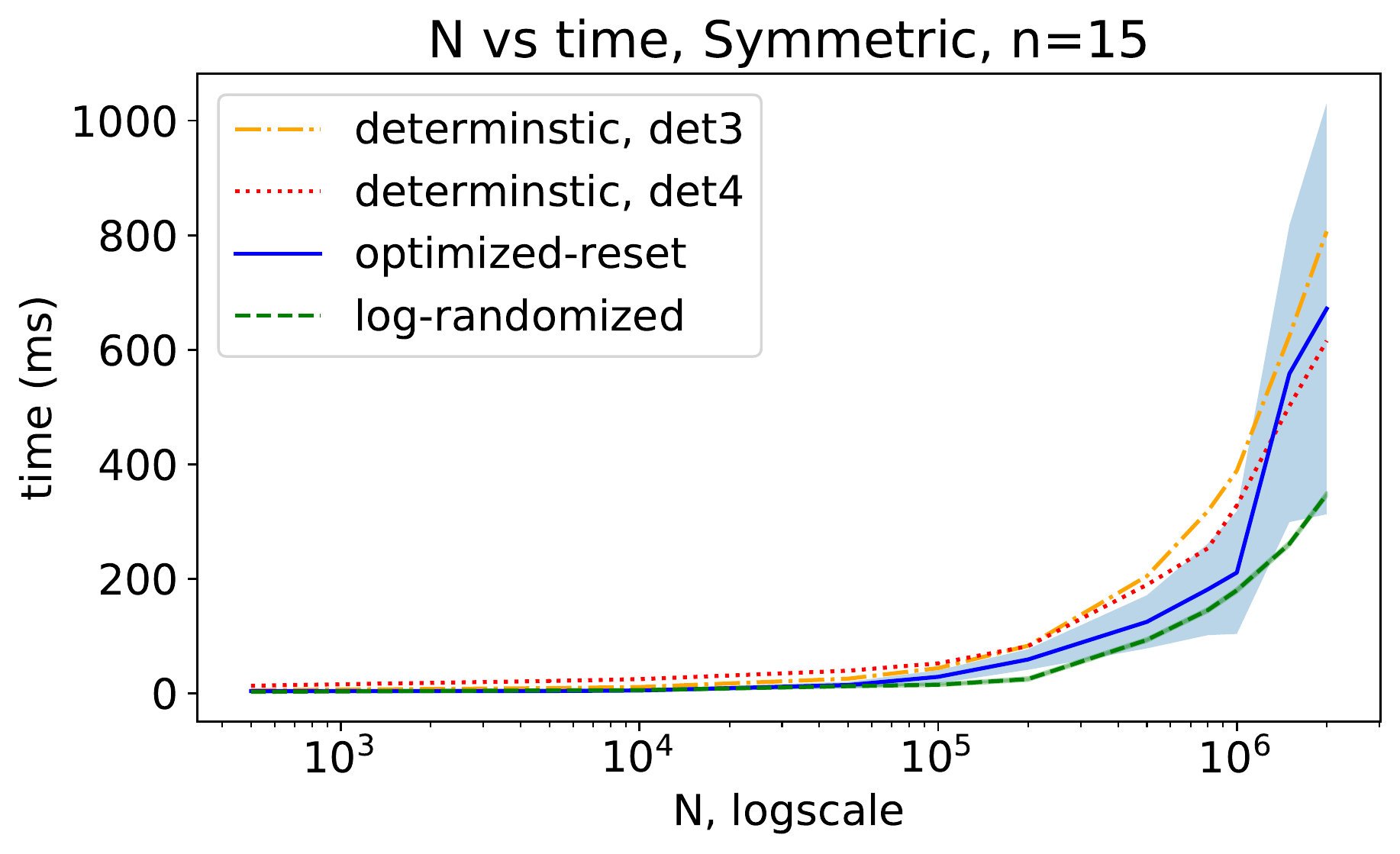}
        \includegraphics[height=2.8cm,width=0.32\textwidth]{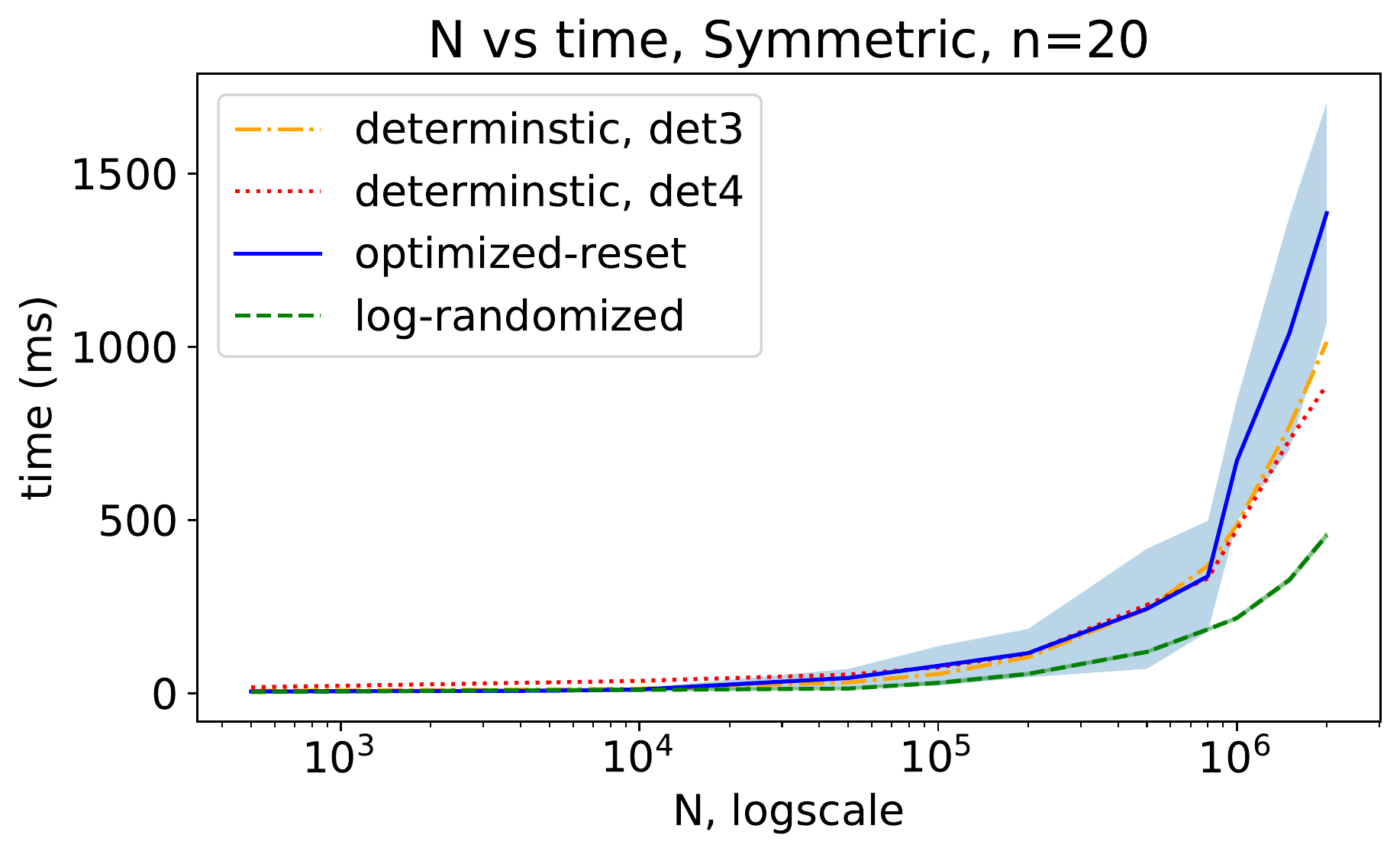}
        \includegraphics[height=2.8cm,width=0.32\textwidth]{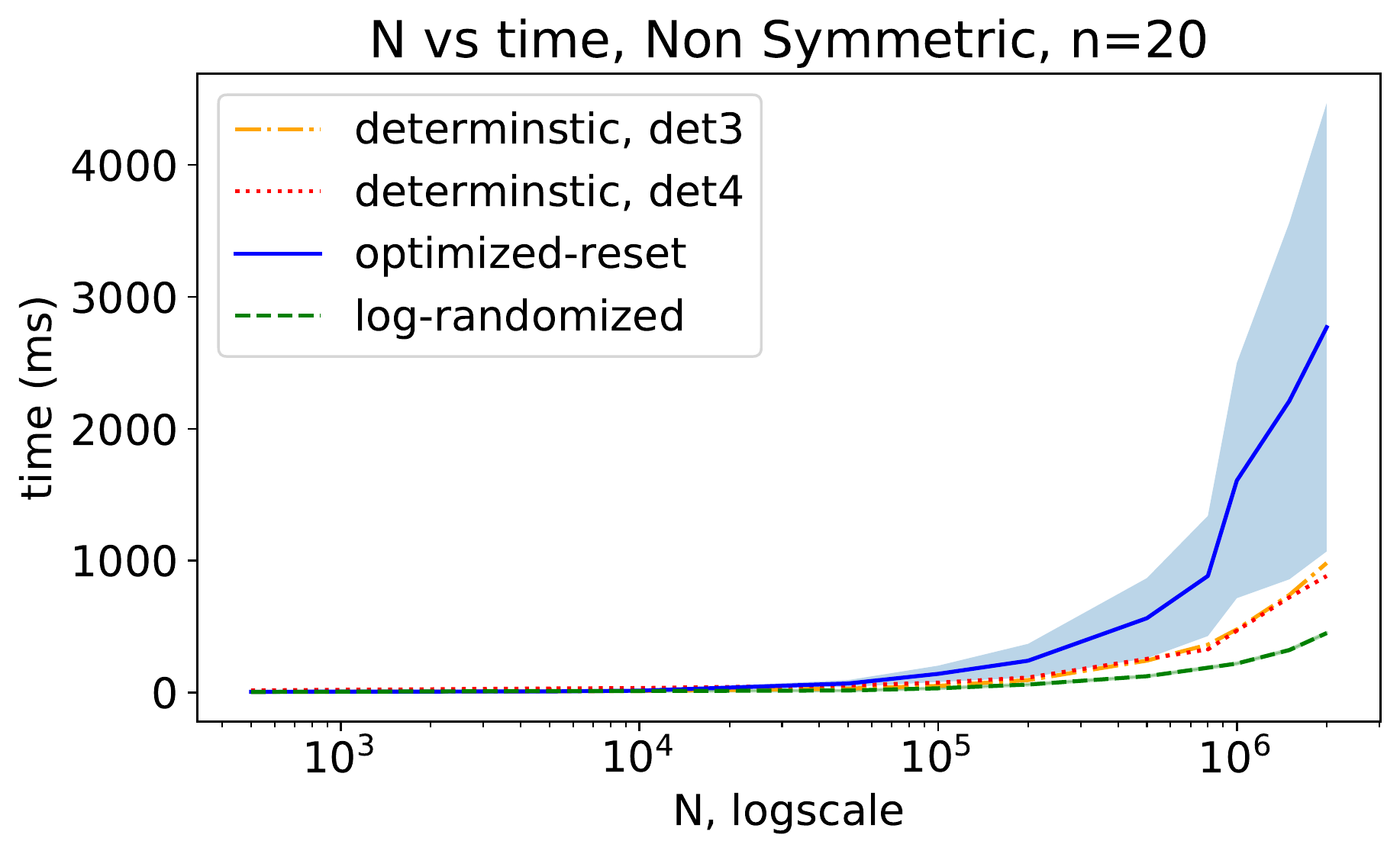}
        \includegraphics[height=2.8cm,width=0.32\textwidth]{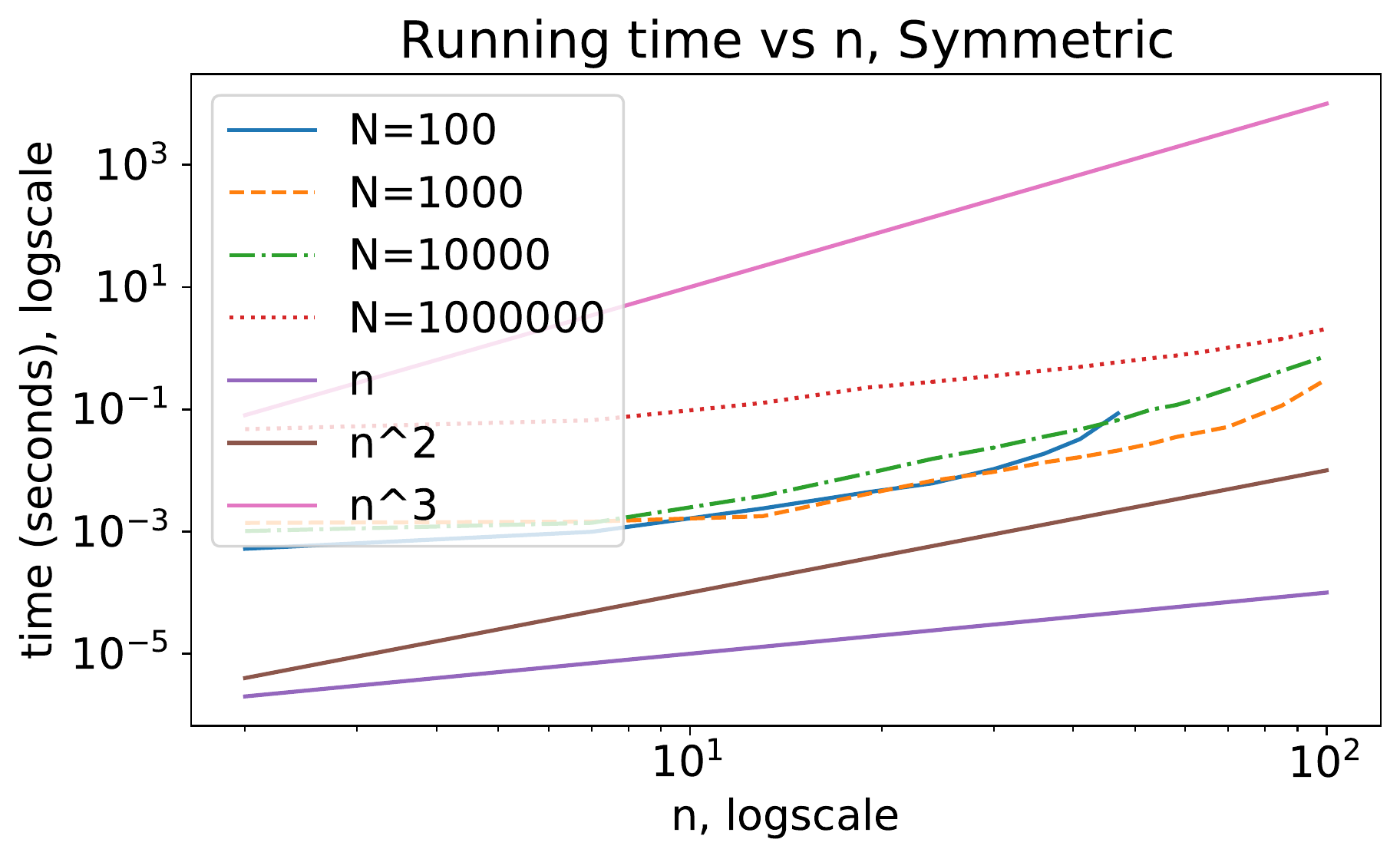}
        \includegraphics[height=2.8cm,width=0.32\textwidth]{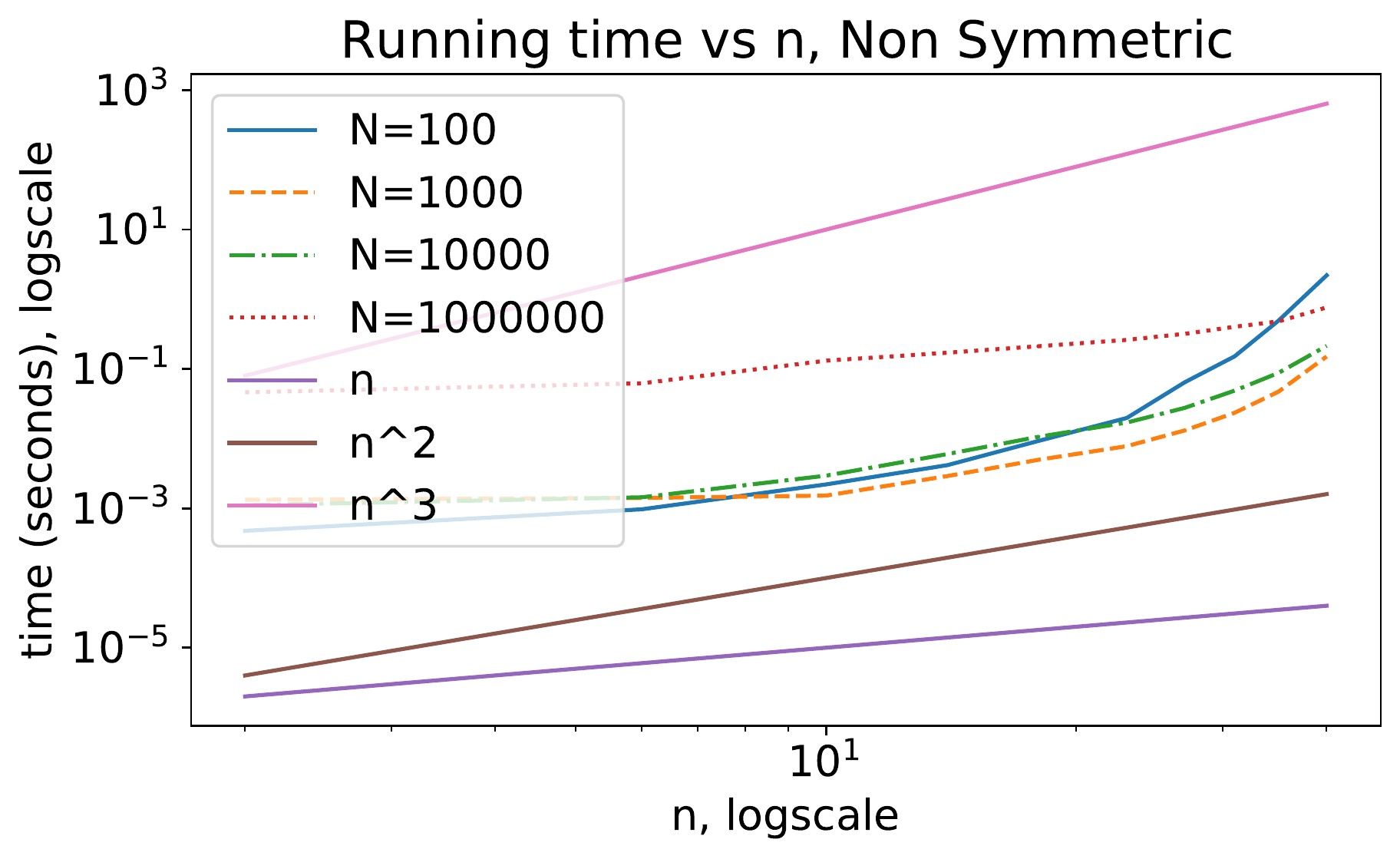}
        \caption{The top row compares the running time of randomized Algorithms against deterministic algorithms as $N$ varies for \textit{symmetric1} (left), \textit{symmetric2} (middle) and \textit{non-symmetric} (right).
        The shaded area represents the standard deviation (from $70$ repetitions of the experiment).
        The bottom row running time of the \textit{log-optimized} algorithm as $n$ varies for different $N$ (average of $70$ samples).}
    \label{fig:Synthetic}
\end{figure}

As expected, the regime when the number of points $N$ is much higher than the number of dimensions $n$, yields the best performance for the randomized algorithm; moreover, Figure \ref{fig:Synthetic} shows that the run time is approximately ${O}(n)$ (in contrast to the runtime of \textit{det4} and \textit{det3} that is $O(n^4)$ resp. $O(n^3)$).
\paragraph{Fast mean-square solvers.}\label{sec:comparison}
In \cite{Maalouf2019} a measure reduction was used to accelerate the least squares method, i.e.~the solution of the minimization problem $ \min_w\|\bX w-\bY\|^2 $ where $\bX\in\R^{N\times d}$ and $\bY\in\R^{N}$; as in Proposition~\ref{th:robust}, $\bX$ denotes a matrix which has as row vectors the elements of $\bx$, similarly for $\bY$.
Sometimes precise solutions are required and in this case, the measure reduction approach yields a scalable method: 
Theorem \ref{th:cath} guarantees the existence of a subset of $n+1=(d+1)(d+2)/2+1$ points $(\bx^\star,\by^\star)$ of $\bx$ and $\by$ such that 
$(\bX|\bY)^\top\cdot (\bX|\bY) = (\bX^\star|\bY^\star)^\top\cdot(\bX^\star|\bY^\star)$ where we denote with $(\bX|\bY)$ the element of $\R^{N\times (d+1)}$ formed by adding $\bY$ as a column.
However, this implies that $ \|\bX w-\bY\|^2 = \|\bX^\star w-\bY^\star\|^2$ for every $w$, hence it is sufficient to solve the least square problem in much lower dimensions once $\bX^\star$ and $\bY^\star$ have been found. 
We use the following datasets from \cite{Maalouf2019}%
\begin{enumerate*}[label=(\roman*)]
\item\label{itm:roads} 3D Road Network \cite{3ddata} that contains $434874$ records and use the two attributes longitude and latitude to predict height,
\item\label{itm:household} Household power consumption \cite{electrdata} that contains $2075259$ records and use the two 2 attributes active and reactive power to predict voltage.
We also add a synthetic dataset, 
\item\label{itm:exp_synt} where $\bX\in \R^{N\times n}$, $\theta \in \R^{n}$ and $\epsilon \in \R^{N}$ are normal random variables, and $\bY= \bX\theta + \epsilon$ which allows to study various regimes of $N$ and $n$. %
\end{enumerate*}
Figure~\ref{fig:comparison_real_dataset} shows the performance of Algorithm \ref{euclid_opt} with the Las Vegas reset, with the Las Vegas reset and the divide and conquer optimization on the datasets~\ref{itm:roads},\ref{itm:household}.
We observe that already Algorithm \ref{euclid_opt} with Las Vegas resets is on average faster but the running time distribution has outliers where the algorithm takes longer than for the deterministic run time algorithms; combined with divide and conquer the variance is reduced by a lot. 
Figure~\ref{fig:Synthetic_cov} shows the results on the synthetic dataset~\ref{itm:exp_synt} for various values of $N$ and $n=(d+1)(d+2)/2$. %
 \begin{figure}[tbh!]
    \centering
        \includegraphics[height=2.8cm,width=0.48\textwidth]{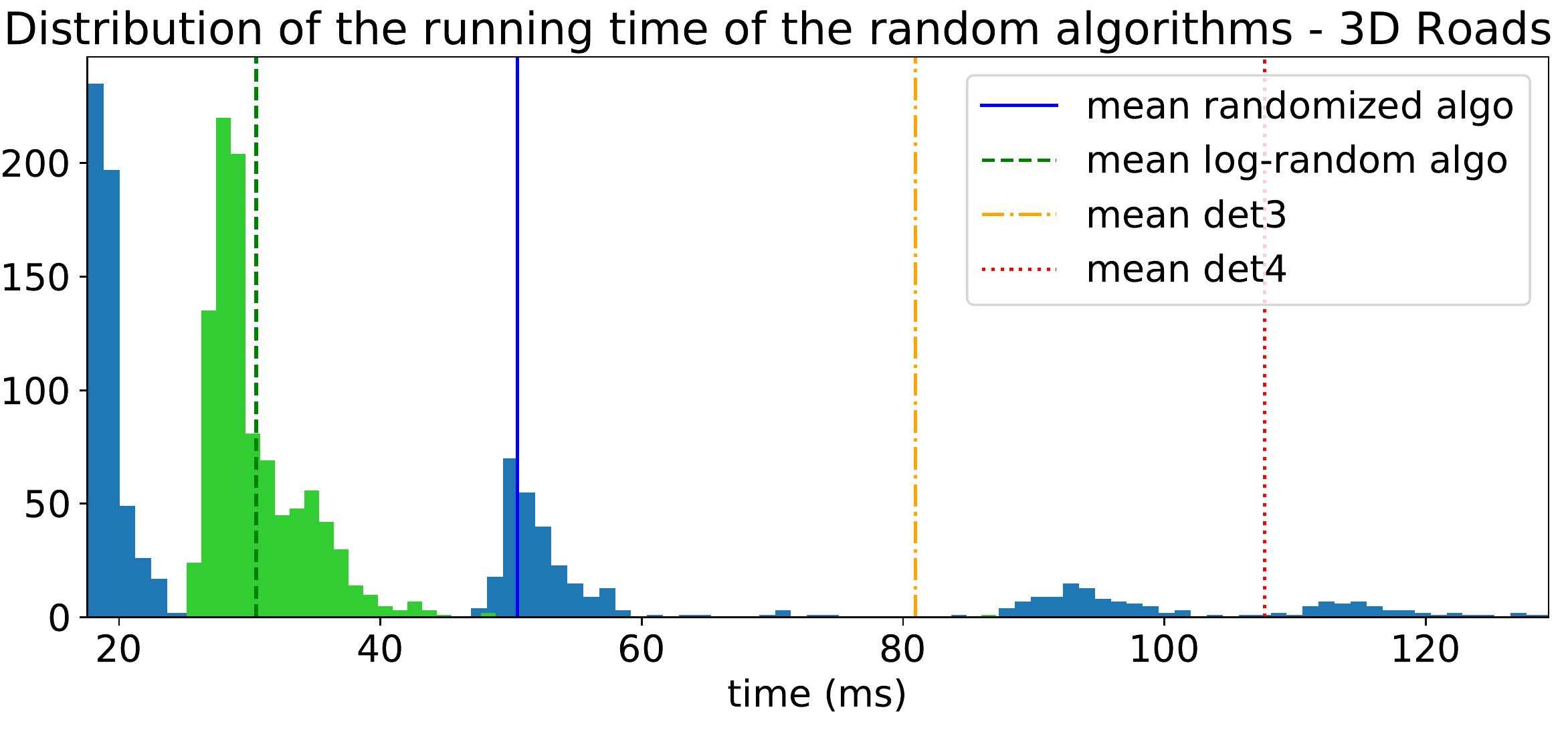}
        \includegraphics[height=2.85cm,width=0.48\textwidth]{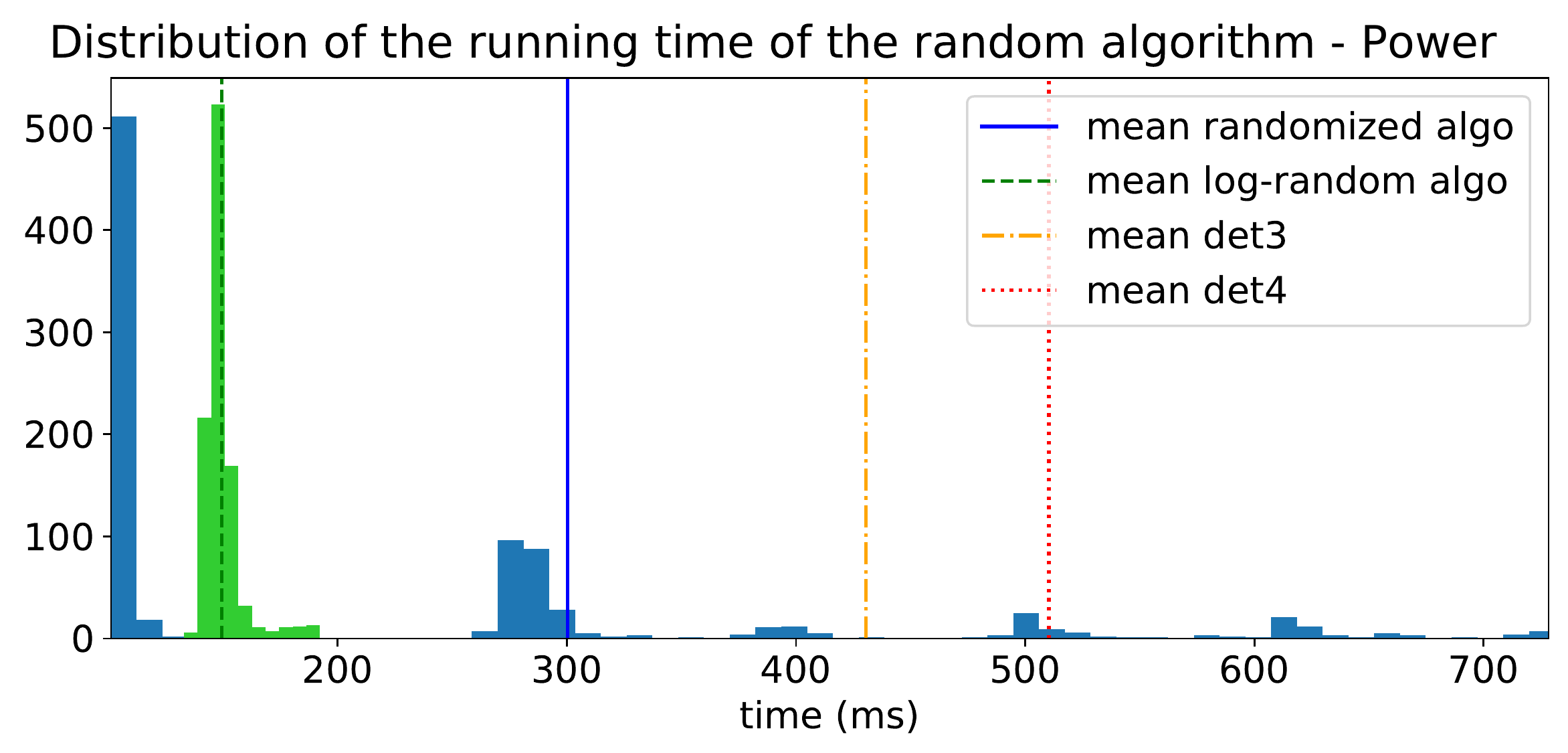}
        \caption{Histogram of the running time of Algorithm~\ref{euclid_opt} with a reset strategy and of the ``divide and conquer'' variation algorithm (\textit{log-random}).
          The vast majority of probability mass of the random runtime is below any of runtimes of the deterministic algorithms although with small probability it can take longer than the deterministic runtimes.} 
  	
    \label{fig:comparison_real_dataset}
\end{figure}
\begin{figure}[tbh!]
        \centering
        \includegraphics[height=2.8cm,width=0.3\textwidth]{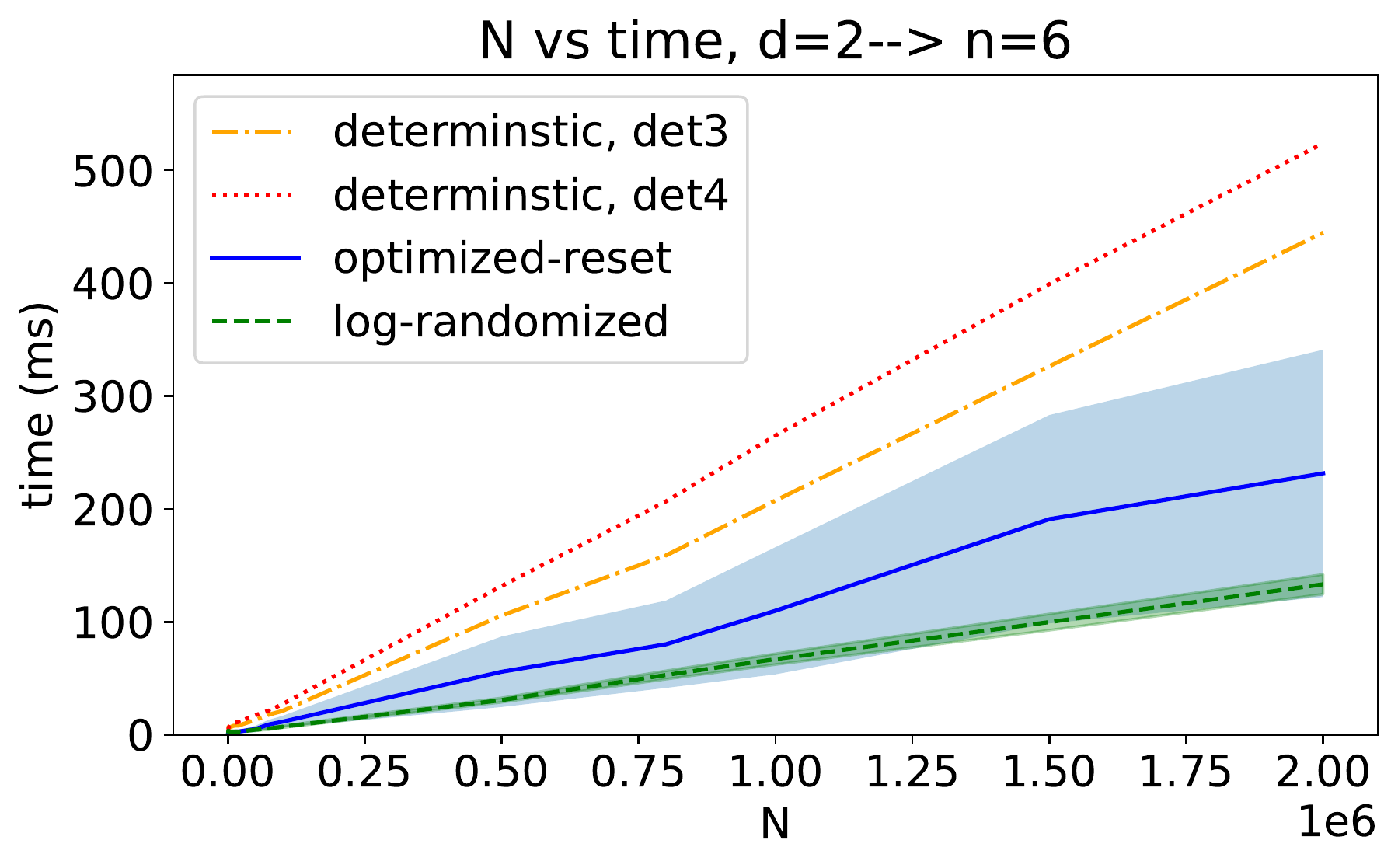}
        \includegraphics[height=2.8cm,width=0.3\textwidth]{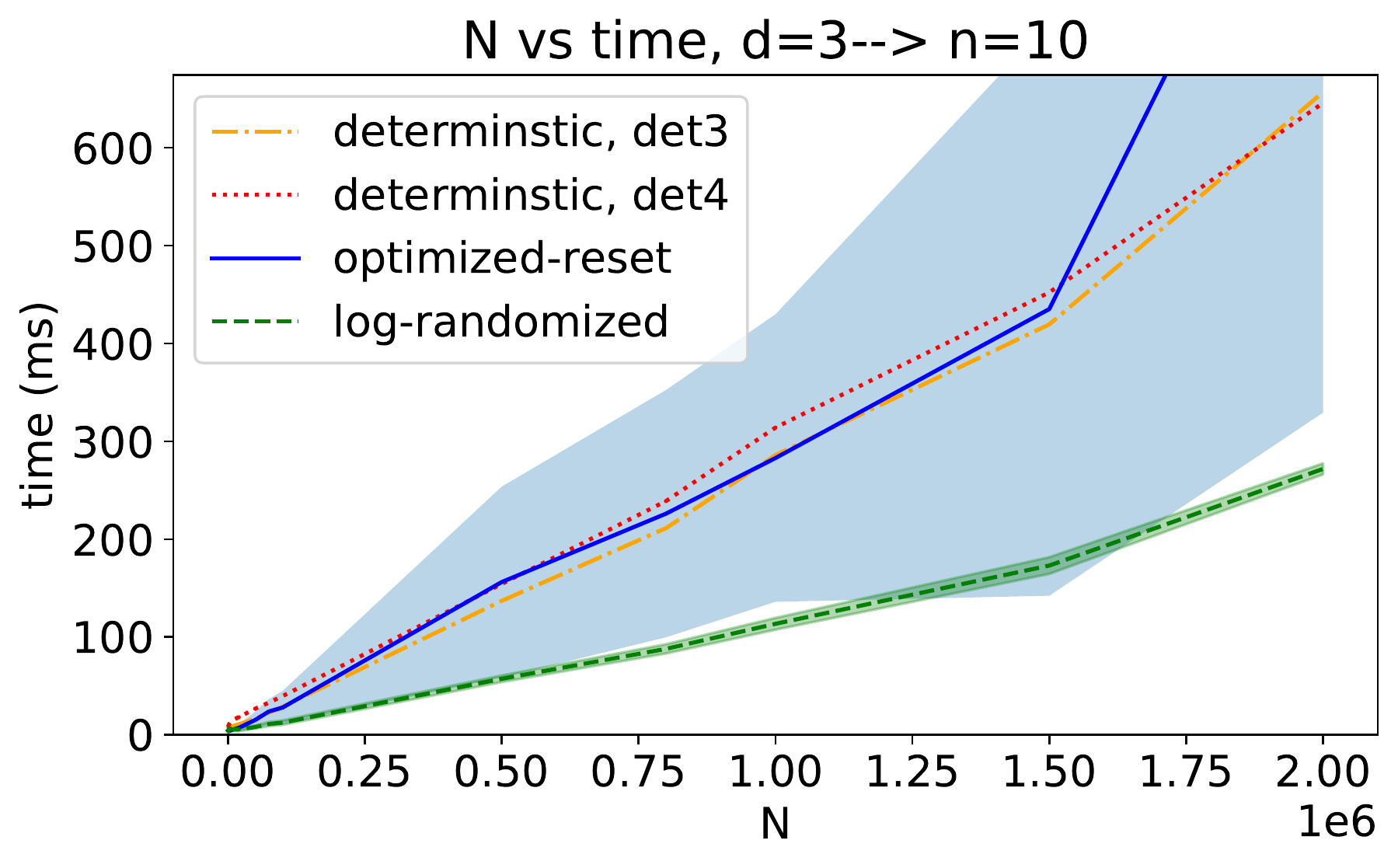}
        \includegraphics[height=2.8cm,width=0.3\textwidth]{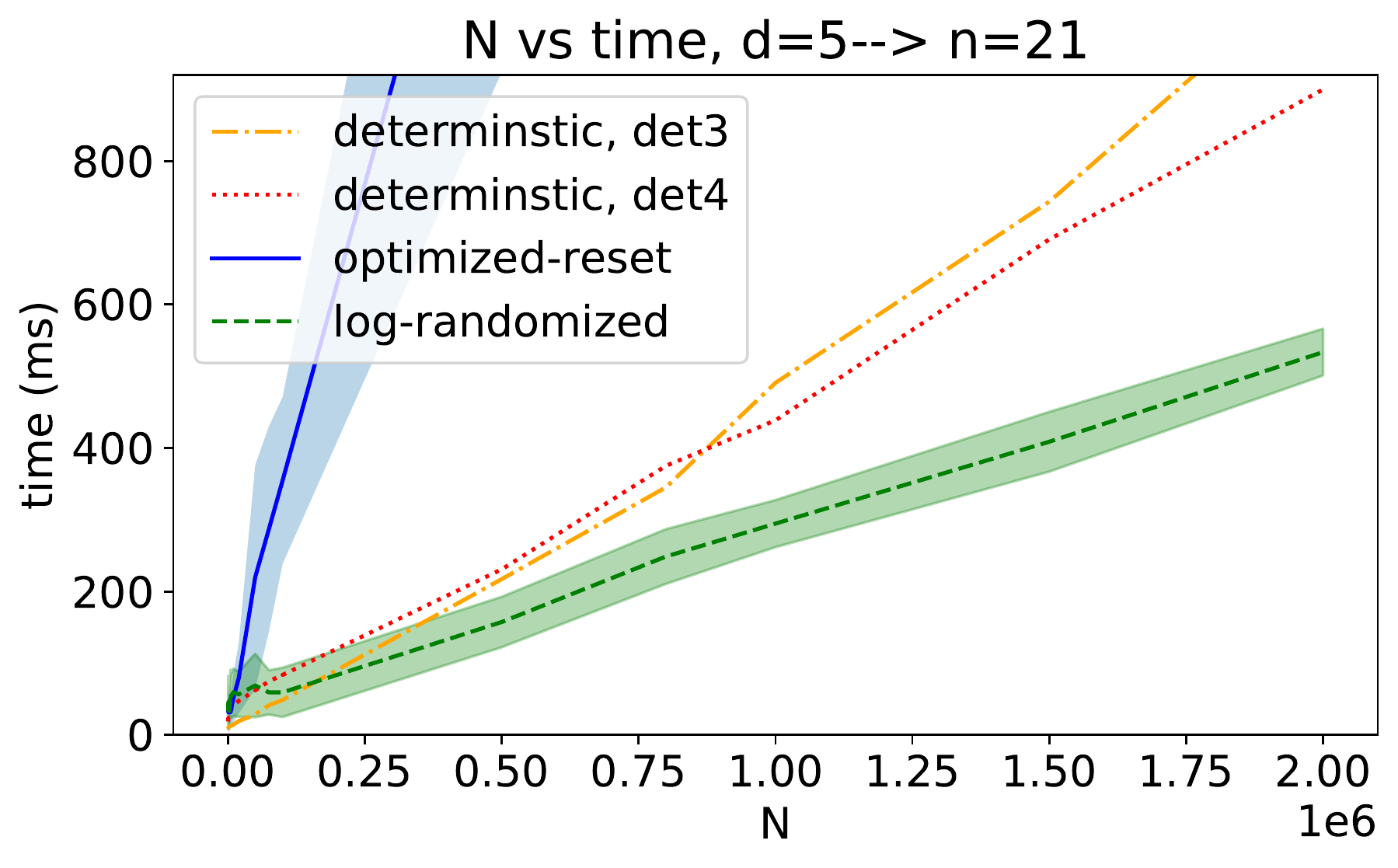}
    \caption{Performance of the different algorithms on the synthetic data set \ref{itm:exp_synt} for various values of $N$ and $n=(d+1)(d+2)/2$.}
    \label{fig:Synthetic_cov}
\end{figure}
\paragraph{Breaking the randomized Algorithm.}
As the above experiments show, Algorithm~\ref{euclid_opt} can lead to big speed ups. 
However, it is not uniformly better than the deterministic algorithms, and there are situations where one should not use it: firstly, Algorithm 2 was optimized to work well in the $N \gg n$ regime and while this is an important regime for data science, the recombination problem itself is of interest also in other regimes \cite{Litterer2012}. %
Secondly, the Las Vegas resets give a finite running time, but it is easy to construct examples where this can be much worse than the deterministic algorithms. Arguably the most practically relevant issue is when the independence hypothesis of Theorem \ref{th:main} is not satisfied. 
This can appear in data sets with a high number of highly correlated categorical features, such as \cite{housesalesdata}. 
This can be overcome by using the Weyl Theorem, {see Remark~\ref{rem:drawbacks} in Appendix~\ref{app:algo 1}} but the computational cost 
is higher than computing {the inverse of the cone basis ($A$ in Theorem~\ref{th:main})} and the benefits would be marginal, if not annulled, compared to the deterministic algorithms.
More relevant is that Algorithm~\ref{euclid_opt} can be easily combined with any of the deterministic algorithms to build an algorithm that has a worst case run time of the same order as a deterministic one but has a good chance of being faster; see Appendix~\ref{sec:combine} for details and experiments.

\section{Summary}
We introduced a randomized algorithm that reduces the support of a discrete measure supported on $N$ atoms down to $n+1$ atoms while preserving the statistics as captured with $n$ functions.
The key was a characterization of the barycenter in terms of negative cones, that inspired a greedy sampling.
Motivated by the geometry of cones this greedy sampling can be optimized, and finally combined with optimization methods for randomized algorithms.     
This yields a ``greedy geometric sampling'' that follows a very different strategy than the previous deterministic algorithms, and that performs very well in the big data regime when $N \gg n$ as is often the case for large sample sizes common in inference tasks such as least square solvers.

\clearpage

\section*{Broader Impact}
The authors do not think this section is applicable to the present work, this work does not present any foreseeable societal consequence.

\section*{Acknowledgements and Disclosure of Funding}
The authors want to thank The Alan Turing Institute and the University of Oxford for the financial support given. FC is supported by The Alan Turing Institute, TU/C/000021, under the EPSRC Grant No. EP/N510129/1. HO is supported by the EPSRC grant ``Datasig'' [EP/S026347/1], The Alan Turing Institute, and the Oxford-Man Institute.

\addcontentsline{toc}{section}{References}
\bibliographystyle{unsrt}
\bibliography{Biblio}

\begin{thebibliography}{10}

\bibitem{Tchak}
V.~Tchakaloff.
\newblock Formules de cubature m\'ecanique \`a coefficients non n\'egatifs.
\newblock {\em Bulletin des Sciences Math\'ematiques}, 81:123--134, 1957.

\bibitem{Bayer2006}
Christian Bayer and Josef Teichmann.
\newblock The proof of tchakaloff’s theorem.
\newblock {\em Proceedings of the American Mathematical Society},
  134(10):3035--3041, oct 2006.

\bibitem{Davis1967}
Philip~J. Davis.
\newblock A construction of nonnegative approximate quadratures.
\newblock {\em Mathematics of Computation}, 21:578--582, 1967.

\bibitem{Litterer2012}
Christian Litterer, Terry Lyons, et~al.
\newblock High order recombination and an application to cubature on wiener
  space.
\newblock {\em The Annals of Applied Probability}, 22(4):1301--1327, 2012.

\bibitem{maria2016a}
Maria Tchernychova.
\newblock {\em Caratheodory cubature measures}.
\newblock PhD thesis, University of Oxford, 2016.

\bibitem{Maalouf2019}
Alaa Maalouf, Ibrahim Jubran, and Dan Feldman.
\newblock Fast and accurate least-mean-squares solvers.
\newblock In {\em Advances in Neural Information Processing Systems}, pages
  8305--8316, 2019.

\bibitem{piazzon2017caratheodory}
Federico Piazzon, Alvise Sommariva, and Marco Vianello.
\newblock Caratheodory-tchakaloff subsampling.
\newblock {\em Dolomites Research Notes on Approximation}, 10(1), 2017.

\bibitem{hayakawa2020monte}
Satoshi Hayakawa.
\newblock Monte carlo cubature construction.
\newblock {\em arXiv preprint arXiv:2001.00843}, 2020.

\bibitem{agarwal2005geometric}
Pankaj~K Agarwal, Sariel Har-Peled, and Kasturi~R Varadarajan.
\newblock Geometric approximation via coresets.
\newblock {\em Combinatorial and computational geometry}, 52:1--30, 2005.

\bibitem{phillips2016coresets}
Jeff~M Phillips.
\newblock Coresets and sketches.
\newblock {\em arXiv preprint arXiv:1601.00617}, 2016.

\bibitem{Huggins2016CoresetsFS}
Jonathan~H. Huggins, Trevor Campbell, and Tamara Broderick.
\newblock Coresets for scalable bayesian logistic regression.
\newblock {\em ArXiv}, abs/1605.06423, 2016.

\bibitem{Feldman2013TurningBD}
Dan Feldman, Melanie Schmidt, and Christian Sohler.
\newblock Turning big data into tiny data: Constant-size coresets for k-means,
  pca and projective clustering.
\newblock In {\em SODA}, 2013.

\bibitem{Cosentino2020a}
Francesco Cosentino, Harald Oberhauser, and Alessandro Abate.
\newblock Carathéodory sampling for stochastic gradient descent.
\newblock {\em arXiv preprint arXiv:2006.01819v2}, 2020.

\bibitem{Schneider2008}
Rolf Schneider and Wolfgang Weil.
\newblock {\em Stochastic and integral geometry}.
\newblock Probability and its Applications (New York). Springer-Verlag, Berlin,
  2008.

\bibitem{Sommerville1927}
D.~M.~Y. Sommerville.
\newblock The relations connecting the angle-sums and volume of a polytope in
  space of n dimensions.
\newblock {\em Proceedings of the Royal Society of London. Series A, Containing
  Papers of a Mathematical and Physical Character}, 115(770):103--119, 1927.

\bibitem{Bollobas1997}
B\'{e}la Bollob\'{a}s.
\newblock Volume estimates and rapid mixing.
\newblock In {\em Flavors of geometry}, volume~31 of {\em Math. Sci. Res. Inst.
  Publ.}, pages 151--182. Cambridge Univ. Press, Cambridge, 1997.

\bibitem{Clarkson2010}
Kenneth~L. Clarkson.
\newblock Coresets, sparse greedy approximation, and the {F}rank-{W}olfe
  algorithm.
\newblock {\em ACM Transactions on Algorithms}, 6(4):Art. 63, 30, 2010.

\bibitem{Aljundi2019}
Rahaf Aljundi, Min Lin, Baptiste Goujaud, and Yoshua Bengio.
\newblock Gradient based sample selection for online continual learning.
\newblock {\em Advances in Neural Information Processing Systems}, March 2019.

\bibitem{Regis2016}
Rommel~G. Regis.
\newblock On the properties of positive spanning sets and positive bases.
\newblock {\em Optimization and Engineering. International Multidisciplinary
  Journal to Promote Optimization Theory \& Applications in Engineering
  Sciences}, 17(1):229--262, 2016.

\bibitem{Luby1993}
Michael Luby, Alistair Sinclair, and David Zuckerman.
\newblock Optimal speedup of {L}as {V}egas algorithms.
\newblock {\em Information Processing Letters}, 47(4):173--180, 1993.

\bibitem{3ddata}
3d road network, north jutland, denmark.
\newblock
  \href{https://archive.ics.uci.edu/ml/datasets/3D+Road+Network+(North+Jutland,+Denmark)}{https://archive.ics.uci.edu/ml/datasets/}.

\bibitem{electrdata}
Individual household electric power consumption.
\newblock
  \href{https://archive.ics.uci.edu/ml/datasets/individual+household+electric+power+consumption}{https://archive.ics.uci.edu/ml/datasets/}.

\bibitem{housesalesdata}
House sales in king county, usa.
\newblock
  \href{https://www.kaggle.com/harlfoxem/housesalesprediction}{https://www.kaggle.com/harlfoxem/housesalesprediction}.

\bibitem{Ziegler1995}
G\"{u}nter~M. Ziegler.
\newblock {\em Lectures on polytopes}, volume 152 of {\em Graduate Texts in
  Mathematics}.
\newblock Springer-Verlag, New York, 1995.

\bibitem{Wendel1962}
James~G Wendel.
\newblock A problem in geometric probability.
\newblock {\em Math. Scand}, 11:109--111, 1962.

\end{thebibliography}

\newpage
\appendix

\section{Properties of Algorithm~\ref{euclid}}\label{app:algo 1}

\paragraph{Background on polytopes.}
To prepare the proof of Theorem \ref{th:main} we want to recall that a well-known tool from discrete geometry, is that polygons and polyhedral descriptions are equivalent.
That is $C(\bx)$ is an affine span of the vectors $\bx$ as in Definition~\ref{def:cone}, but equivalently $C(\bx)$ is the intersection of hyperplanes.
In general this is the content of the celebrated Weyl--Minkowksi theorem and computing one representation from the other is non-trivial, see~\cite{Ziegler1995}.
However, when restricted to $n$ generic vectors in $\R^n$ (as is the case required in Theorem~\ref{th:main}), one can immediately switch from one to the other, see item \ref{itm: A and H >} of Theorem \ref{th:main}. 
\paragraph{Proofs of Theorem~\ref{th:main} and Proposition~\ref{prop:worst}}
\thmmain*
\begin{proof}
For item~\ref{itm:weyl} first note that the $h_i$ are well-defined since any set of $n-1$ independent points determines a hyperplane that includes $0$ and that divides $\R^n$ into two parts.
Each of these two parts is of the form $\{x: \langle h,x \rangle\le 0\}$ or $\{x: \langle h,x \rangle>0\}$ and the additional condition $\langle h,x_i \rangle<0$ selects one of the two parts.
Now let $c=\sum_{x \in \bx\setminus\{x_{n+1}\}} w_x x $ be a general vector.
Since $H_{\bx\setminus \{x_{n+1}\}} x \le 0$, for $x\in \bx\setminus\{x_{n+1}\}$, it follows that $H_{\bx\setminus \{x_{n+1}\}} c = \sum_{x \in \bx\setminus\{x_{n+1}\}} w_x H_{\bx\setminus\{x_{n+1}\}} x$ satisfies $H_{\bx\setminus\{x_{n+1}\}} c\leq 0$ if and only if the $w_x$, for $x \in \bx\setminus\{x_{n+1}\}$, are positive.

For item \ref{itm: A and H >}, we can write $x=\sum_{i=1}^n w_ix_i$, for some $w_i\in\R$, hence $Ax = (w_1,\ldots,w_n)^\top $.
  By definition 
  \[ 
  H_{\bx\setminus \{x_{n+1}\}} x = \sum w_i H_{\bx\setminus \{x_{n+1}\}}x_i= \sum w_i (\langle h_1, x_i\rangle  ,\ldots, \langle h_n , x_i \rangle)^\top= (w_1\langle h_1, x_1\rangle,\ldots, w_n\langle h_n, x_n\rangle)^\top.
  \]
  Note that $\langle h_i, x_i\rangle \le 0 $ by definition, therefore $\sign\{ -Ax \}  =\sign\{ H_{\bx\setminus\{x_{n+1}\}} x  \}$.
The statement with the reversed inequalities follows similarly. 

For item~\ref{itm: barycenter}$(\Rightarrow)$ assume there exists a convex combination of $\bx$, this means that $x_{n+1}=-\frac{1}{w_{n+1}}\sum_{i=1}^n w_i x_i$, and 
\[ Ax_{n+1}=-\frac{1}{w_{n+1}}\sum_{i=1}^n w_i A x_i= \sum_{i=1}^n -\frac{w_i}{w_{n+1}} e_i.\]
Therefore, $ Ax_{n+1}\leq 0$ which is by item \ref{itm: A and H >} equivalent to $H_{\bx\setminus \{x_{n+1}\}} x \geq 0$. Thus, $x_{n+1} \in C^-(\bx \setminus \{x_{n+1}\})$.
Finally, for item~\ref{itm: barycenter}$(\Leftarrow)$ assume that $x_{n+1} \in C^-(\bx \setminus \{x_{n+1}\})$.
The by item \ref{itm: A and H >} $ Ax_{n+1}\leq 0$.
Moreover, $\exists \lambda_i\in\R$ such that $x_{n+1}=\sum_{i=1}^n \lambda_i x_i$, therefore  $ Ax_{n+1}= \sum_{i=1}^n \lambda_i e_i=(\lambda_1, \ldots, \lambda_n)^\top\leq 0.$
Let us call $\lambda^*:=1-\sum_{i=1}^n \lambda_i$, by the decomposition of $x_{n+1}$ we know that 
\[\frac{1}{\lambda^*}x_{n+1}+\sum_{i=1}^n \frac{-\lambda_i}{\lambda^*} x_i =0 \text{ and } \frac{1}{\lambda^*}+\sum_{i=1}^n \frac{-\lambda_i}{\lambda^*}=1 \text{ and } \frac{1}{\lambda^*}, \frac{-\lambda_i}{\lambda^*}\geq 0.\]
\end{proof}
\begin{remark}\label{rem:drawbacks}
The assumption that $ \{\bx\setminus \{x_{n+1}\}\}$ spans $\R^n$ can be relaxed, indeed item~\ref{itm:weyl} is a particular case of the Weyl's Theorem, which briefly does not require the independence of the cone basis. \\
From an implementation point of view, however, item~\ref{itm: A and H >} gives an important boost, indeed the computation of $H_{\bx\setminus\{x_{n+1}\}}$ is heavier than inverting a matrix, i.e. computing $A$, since it requires the computation of the coefficients of $n$ different hyperplanes in $\R^n$. Moreover, speaking about the greedy searching strategy of Algorithm~\ref{euclid_opt}, using $H_{\bx\setminus\{x_{n+1}\}}$ in place of $A$ does not allow the use of the Sherman–Morrison formula weighing even more on the total computational cost.
\end{remark}
\propworst*
\begin{proof}
  Note that in every run through the loop, $n$ points are randomly selected.
  However, by Theorem~\ref{th:main}, Algorithm~\ref{euclid} finishes when the event 
\begin{align}\label{eq:event}
A := \{&\text{for $n$ uniformly at random chosen points $\bx^\star$ from $\bx$, }\exists x\in \bx\,\, s.t. \,\,x\in C^-(\bx^\star) \}.
\end{align}
occurs. 
Combined with Tchakaloff's Theorem guarantees this shows that there exists at least one set of $n$ points $\bx^\star$ such that $C^-(\bx^\star)\neq \emptyset$ and therefore
\begin{align}
\Prob(A) \geq \frac{{n+1\choose n}}{{N\choose n}} = \frac{n\cdot n!(N-n)! }{N!},
\end{align}
By independence, $\tau$ can be modelled by a geometric distribution with parameter $p = \Prob(A)$
\[
  \Prob(\tau=k) = (1-p)^{k-1}p 
\]  
and the bounds for mean and variance follow.\\
\end{proof}
\section{Properties when applied to special measures}\label{ap:empirical}
\paragraph{Proof of Proposition~\ref{prop:empirical}}
\propcomplexity*
\begin{proof}
As in Proposition \ref{prop:worst}, the algorithm terminates when the event 
\begin{align}
A:=\{\text{for $n$ uniformly at random chosen points $\bx^\star$ from $\bx$, }\exists x\in \bx\,\, s.t. \,\,x\in C^-(\bx^\star)\}
\end{align}
happens, where $\bx=\{F(X_1), F(X_2),\ldots,F(X_N)\}$.\\
For item~\ref{itm:tau general} denote $F_i:   =(f_1(X_{I_i}),\ldots,f_n(X_{I_i}))$ where $\{I_1,\ldots,I_N\}$ is a uniform shuffle of $\{1,\ldots,N\}$, i.e. a random permutation of its elements that makes every rearrangement equally probable,
  then
  \begin{align}
    A=\{\exists i\in \{n+1,\ldots,N\} \text{ s.t. } F_i \in C^{-}(F_1,\ldots,F_n)\}
  \end{align}
  and note that 
\begin{align}
  \Prob(A|& E)= \frac{\Prob(E| A) \Prob(A)}{\Prob(E)} \ge \Prob{(A)}
\end{align} 
since by Theorem \ref{th:main} $\Prob{(E|A)} = 1$  
and $\Prob(E)>0$ since $N\geq n+1$ and $\EE F(X)=0$.
The estimate of Proposition \ref{prop:worst} $ \Prob(E|0\in \operatorname{Conv}\{F_i\}) \geq \frac{n\cdot n!(N-n)!}{N!} $ is still valid, moreover
\begin{align}
 \Prob(A|E)\geq\Prob(A)=&\Prob(\exists i\in\{n+1,\ldots,N\}\text{ such that }F_{i}\in C^{-}(F_{1},\ldots,F_{n}))\\=&1-\prod_{j=n+1}^{N}\Prob(F_{j}\notin C^{-}(F_{1},\ldots,F_{n}))\\=&1-\Prob(F_{n+1}\notin C^{-}(F_{1},\ldots,F_{n}))^{N-n}\\=&1-\Prob(0\notin\operatorname{Conv}(F_{1},\ldots,F_{n+1}))^{N-n}
\end{align}
where the last equality follows from Theorem~\ref{th:main}.
We have therefore two different bounds for $\Prob(A|E)$, so we can take the maximum, i.e. 
\begin{align}
\Prob(A|E)\geq \max \left\{ \frac{n\cdot n!(N-n)!}{N!},1-\Prob\left(0\not\in\operatorname{Conv}\{F_{1},\ldots,F_{n+1}\}\right)^{N-n} \right\}
\end{align}
Item \ref{itm:tau symm}. In \cite{Wendel1962} the author shows that when the $F_i$ are distributed uniformly randomly on the unit sphere, then $\Prob(0 \notin \operatorname{Conv}(F_1,\ldots,F_{n+1}))^{N-n} = \left(1-2^{-n}\right)^{N-n}$. In \cite{Schneider2008}[Theorem 8.2.1] it is shown the same result, for all the symmetric distributions with respect to 0.\\
Now $\tau$ can be modelled by a geometric distribution with parameter $p=\Prob(A|E)$, i.e. $\Prob(\tau =i ) \ge (1-p)^{i-1}p$ and the mean and variance follows.  \\
For item \ref{itm:N_to_infty}, it is enough to show $\Prob(A)\to 1$:
\begin{align}
\Prob(A)=\Prob(A\big|E)\times\Prob(E)+\Prob\left(A\big|E^{C}\right)\times\Prob\left(E^{C}\right)=&\Prob(A\big|E)\times\Prob(E)+0\times\Prob\left(E^{C}\right)\to1\times1,
\end{align}
as $N\to\infty$, where $\Prob\left(A\big|E^{C}\right)=0$ is due to Theorem \ref{th:main}, the convergence $\Prob(E)\to 1$ is guaranteed by Theorem \ref{th:convergence_CH} and the convergence $\Prob(A\big|E)\to 1$, is guaranteed by the proof of item \ref{itm:tau general} for fixed $n$.
\end{proof}
As intuition suggest, the event that the mean is included in the convex hull occurs almost surely. 
\begin{theorem}[\cite{hayakawa2020monte}]\label{th:convergence_CH}
  Let $X_1,\ldots, X_N$ be i.i.d.~samples from a random variable $X$ that has a first moment $\EE[X]<\infty$.
  Then
$
\Prob(\EE[X] \in \operatorname{Conv}\{X_i\}_{i=1}^N)\to 1,\,\,\text{ as } N\to\infty
$.
\end{theorem}

\paragraph{Sampling from empirical measures}

Often we do not know the distribution $\phi$ of the points or $\EE[F(X)]$, moreover it could be that for the realized samples $\{F_i\}_{i=1}^N $,
$\EE[F(X)]\not\in\operatorname{Conv}\{ F_i\}_{i=1}^N $. 
In these cases, due to Theorem \ref{th:cath}, and since Algorithm \ref{euclid} is based on Theorem \ref{th:main}, which assumes that the barycentre of the points given is $0$,
the input of the Algorithm is not the collection $\{F_i\}_{i=1}^N$, but 
\begin{align}
\hat{F_{i}}=&(f_{1}(X_{i}),\ldots,f_{n}(X_{i}))-\left(\sum_{j=1}^{N}f_{1}(X_{j})w_{j},\ldots,\sum_{j=1}^{N}f_{n}(X_{j})w_{j}\right),
\end{align}
in this way we are sure that the barycentre is $0$ and $0\in \operatorname{Conv}\{ \hat F_i\}_{i=1}^N$, and the hypothesis of both Theorem \ref{th:cath} and \ref{th:main} are satisfied.
Unfortunately the $\{\hat{F_{i}}\}_{i=1}^N$ are not independent, which leads to an impossible analysis, even though the correlation between the $\hat F_i$ decreases when $N$ becomes bigger and tends to $0$. 
Thus, we can believe that the analysis of the proof of Proposition~\ref{prop:empirical} is a good approximation of the complexity of the Algorithm \ref{euclid}, when the $\{\hat{F_{i}}\}_{i=1}^N$ are given as input and $N$ is big ``enough'', as it is shown in Figure \ref{fig:theoretical_result} and Figure \ref{fig:bVSo} in case of symmetric distribution. 

It is relevant to note at this point that we can always consider uniform measures, i.e. $\mu = \frac{1}{N}\sum^N_{i=1} \delta_{x_i}$, modifying the support of the measure, and then eventually go back to the original (not-uniform) measure.
\begin{lemma}\label{lem:trick_sphere} Let us consider a set $\bx=\{x_{i}\}_{i=1}^{N}$in $\mathbb{R}^{n}$ and a sequence $\{\kappa_{i}\}_{i=1}^{N}$ of strictly positive numbers. There exists a measure $\mu$ on $\bx$ such that $\mu(\bx)=0$ if and only if there exists a measure $\mu^{\star}$ on $\{\frac{x_{i}}{\kappa_{i}}\}_{i=1}^{N}$ such that $\mu^{*}(\{\frac{x_{i}}{\kappa_{i}}\}_{i=1}^{N})=0$.
\end{lemma}
\begin{proof} Let us assume that there exists $\mu$ on $\bx$ such that $\mu(\bx)=0$, and let us call $\mu_i:=\mu(x_i)$. It is enough to define $\mu^\star=\mu_i\kappa_i$. The other side of the equivalence is proved in the same way.
\end{proof}
\begin{remark}
Lemma \ref{lem:trick_sphere} is a consequence of the fact that a cone is defined only by the directions of the vectors of the ``basis'', and not from their length.
\end{remark}

Proposition~\ref{prop:empirical} shows us a ``universal strategy'' to explore the space of all the combination of points more efficiently, i.e. choosing the basis of the cone to maximize the probability placed on its inverse.
In other words, ideally we should try to maximize 
\begin{align}\label{eq:univer_strategy}
{
\max_{F_i\in\bx}\Prob\left(F(X) \in C^{-}(F_{1},\ldots,F_{n})\right).
}
\end{align}

\begin{figure}[hbt!]
    \centering
        \includegraphics[height=4cm, width=0.4\textwidth]{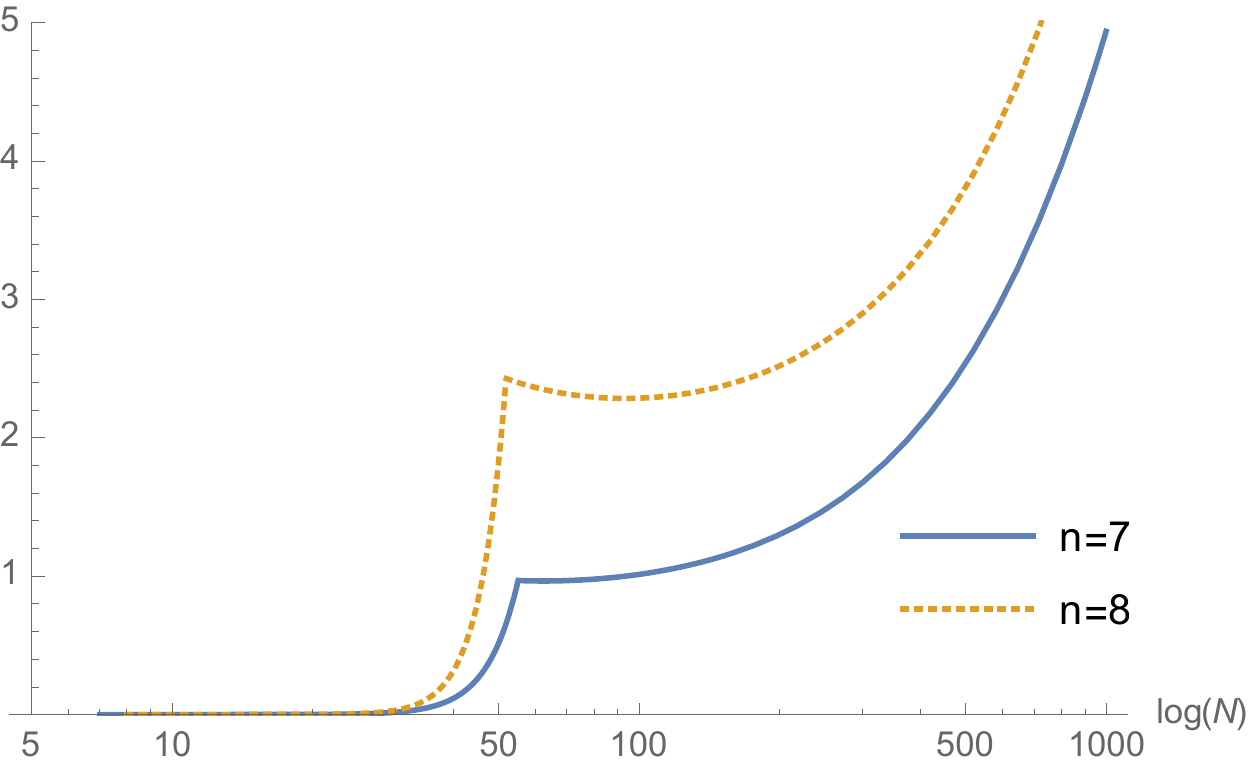}
        \caption{The plots shows the logplot of Equation~\eqref{eq:theoretical_complexity}. It can be seen that this is the same shape obtained with the experimental simulations in Section \ref{sec:exp}.
        }\label{fig:theoretical_result}
\end{figure}

Figure \ref{fig:theoretical_result} shows the complexity of Algorithm \ref{euclid} in case of symmetric distributions; it can be noticed that 
it has a local minima $N_n^*$. %

\section{Properties of Algorithm~\ref{euclid_opt}} \label{app:algo opt} 
\paragraph{Complexity of Algorithm~\ref{euclid_opt}}
\propcomplexityopt*
\begin{proof}
The most expensive steps are the same ones as in Algorithm~\ref{euclid} and in addition the maximization problem.
After the first step, given that we update one point of the basis at time, we can be more efficient using the Sherman–Morrison formula. Let us call $\bX^\star_t$ the matrix whose rows are the vectors of the basis at the step $t$ $\bx^\star_t$, therefore we have $A_t=((\bX^\star_t)^\top)^{-1}$, thus (shifting properly the vectors)
\begin{align}
A_{t+1}=&\left((\bX_{t}^{\star})^{\top}+(X^{\star}-X_{1}^\star)\cdot e_{1}^{\top}\right)^{-1}=A_{t}-\frac{A_{t}(X^{\star}-X_{1}^\star)e_{1}^{\top}A_{t}}{1+e_{1}^{\top}A_{t}(X^{\star}-X_{1}^\star)},
\end{align}
{in the case we want to substitute the ``first'' vector of the basis $X_1^\star$ with the vector $X^{\star}$.} 
Let us note that the only multiplications to be computed are $A_{t}(X^{\star}-X_{1}^\star)$ and $ \left[A_{t}(X^{\star}-X_{1}^\star)\right]\cdot\left[e_{1}^{\top}A_{t}\right]$, which are done in ${O}(n^2)$ operations.
To check if there are points inside the cone (or the inverse cone), we multiply the matrix $A$ times the matrix $\bX$ (of all the remaining vectors $\bX$), and again after the first step costs ${O}(Nn^2)$, we can use the Sherman–Morrison formula as before and obtain
\begin{align}
A_{t+1}\bX^{\top}=&A_{t}\bX^{\top}-\frac{A_{t}(X^{\star}-X_{1}^\star)e_{1}^{\top}A_{t}\bX^{\top}}{1+e_{1}^{\top}A_{t}(X^{\star}-X_{1}^\star)}.
\end{align}
Let us note that we have already computed $A_{t}\bX^{\top}$ at the previous step, $A_{t}(X^{\star}-X_{1}^\star)$ to compute $A_{t+1}$, therefore the only cost is to compute $\left[A_{t}(X^{\star}-X_{1}^\star)\right]\cdot\left[e_{1}^{\top}A_{t}\bX^{\top}\right]$, which is done in ${O}(nN)$ operations.\\
The previous computations show us that after the first step, updating one element at a time, improves the computational efficiency of the successive steps of a factor $n$. %
Let us now tackle the maximization problem, %
it requires to compute the norm, i.e.~${O}(Nn^2)$ operations, which could be done only once.
Moreover, the maximization problem requires to compute (part of) the sum of the vectors in $\bx^\star$ and then the scalar product, which require ${O}(Nn)$. 
The last expensive operation we should consider is due to solve the last system to find the weights. The system we want to solve is  
\begin{align}
\left(\begin{array}{cc}
(\bX^{\star}){}^{\top} & X^{\star}\\
\mathbf{1} & 1
\end{array}\right)&w=\left(\begin{array}{c}
\mathbf{0}\\
1
\end{array}\right),
\end{align}
where $\mathbf{0}$ is a $n\times 1$ vector of $0$, and $\mathbf{1}$ is a $1\times n$ vector of $1$. Using again the Sherman–Morrison formula, since we have already computed $A=(\bX^\star)^{-1}$ the weights $w_i$ can be computed as
\begin{align}
w=&\left(\begin{array}{cc}
(\bX^{\star}){}^{\top} & X^{\star}\\
\mathbf{1} & 1
\end{array}\right)^{-1}\left(\begin{array}{c}
\mathbf{0}\\
1
\end{array}\right)=\left(\begin{array}{cc}
A^{\top}+A^{\top}X^{\star}c^{-1}\mathbf{1}A^{\top}, & -A^{\top}X^{\star}c^{-1}\\
-c^{-1}\mathbf{1}A^{\top}, & c^{-1}
\end{array}\right)\left(\begin{array}{c}
\mathbf{0}\\
1
\end{array}\right)\\=&-\frac{1}{c}\left(\begin{array}{c}
A^{\top}X^{\star}\\
1
\end{array}\right),
\end{align}
where $c=1-\mathbf{1}A^{\top}X^{\star}$ is a number.
In this way we need ${O}(n^2)$ operations, not ${O}(n^3)$, i.e. the complexity of solving a linear system.\\ 
The total cost therefore is
$
O(n^{3}+n^{2}N)+(\kappa-1)O(n^{2}+nN).
$
\end{proof}
\begin{remark}
The gain in the computational cost we obtain using 
the Sherman–Morrison formula has a cost in term of numerical stability.
\end{remark}
\paragraph{Robustness of the solution.}  
\thmrobust*
\begin{proof}
From Theorem~\ref{th:main}, we know that $\hat\bX R+E_{\hat\bx}$ is a solution if and only if $\left(\left(\hat \bX_{-1}R+E_{\hat \bx_{-1}}\right)^{\top}\right)^{-1}\left(\hat X_{1}R+E_{\hat x_{1}}\right)^\top\leq0$.  Let us note that the last product is a vector, therefore we can study the transpose and 
using the Woodbury matrix identity we have that
\begin{align}
\left(\hat{X}_{1}R+E_{\hat{x}_{1}}\right)\left(\hat{\bX}_{-1}R+E_{\hat{\bx}_{-1}}\right)^{-1}=&\left(\hat{X}_{1}+E_{\hat{x}_{1}}R^{-1}\right)\left(\hat{\bX}_{-1}+E_{\hat{\bx}_{-1}}R^{-1}\right)^{-1}\\=&\left(\hat{X}_{1}+E_{\hat{x}_{1}}R^{-1}\right)\left(I-A_{1}^{T}E_{\hat{\bx}_{-1}}\left(I+R^{-1}A_{1}^{T}E_{\hat{\bx}_{-1}}\right)R^{-1}\right)A_{1}^{T}.
\end{align}
Setting the last equation less or equal than $0$ shows the result. 
\end{proof}
This also implies that the solution is invariant under rotations.
\section{Divide and conquer, choice of the subgroup size}\label{sec:choice_div&conq}
As mentioned in Section~\ref{sec:optimized algo}, to apply a divide and conquer strategy requires to balance the size of subgroups against the property of Algorithm~\ref{euclid_opt} to exploit a large number of points as to maximize the likelihood of points being in the (inverse) cone.    
Let us explain how we have chosen $N^*_n=50(n+1)$, which should be thought as linear approximation of the exact minimum for the complexity of Algorithm \ref{euclid_opt}.
\begin{figure}[hbt!]
\centering
        \includegraphics[height=3cm,width=0.3\textwidth]{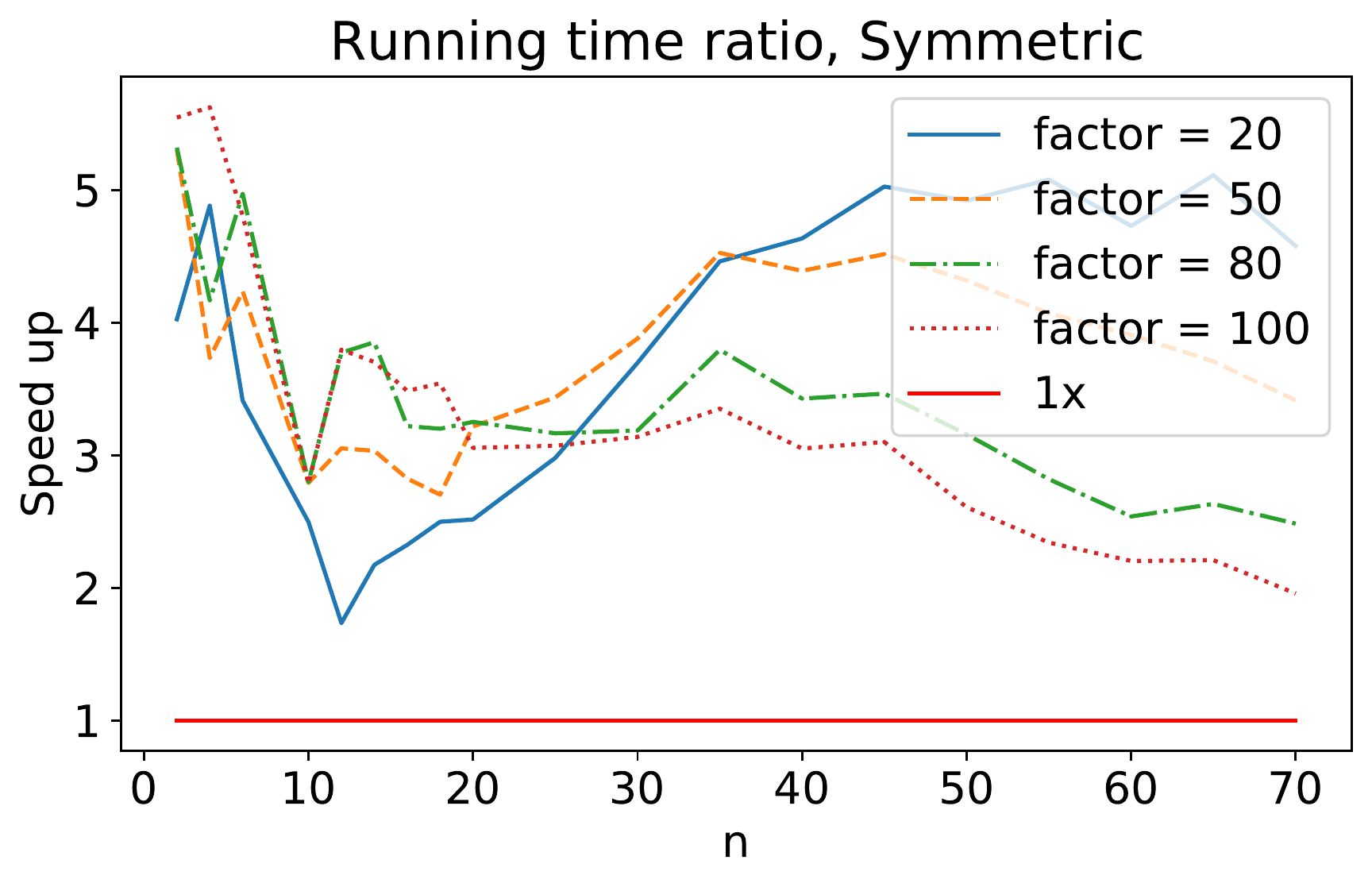}
        \includegraphics[height=3cm,width=0.3\textwidth]{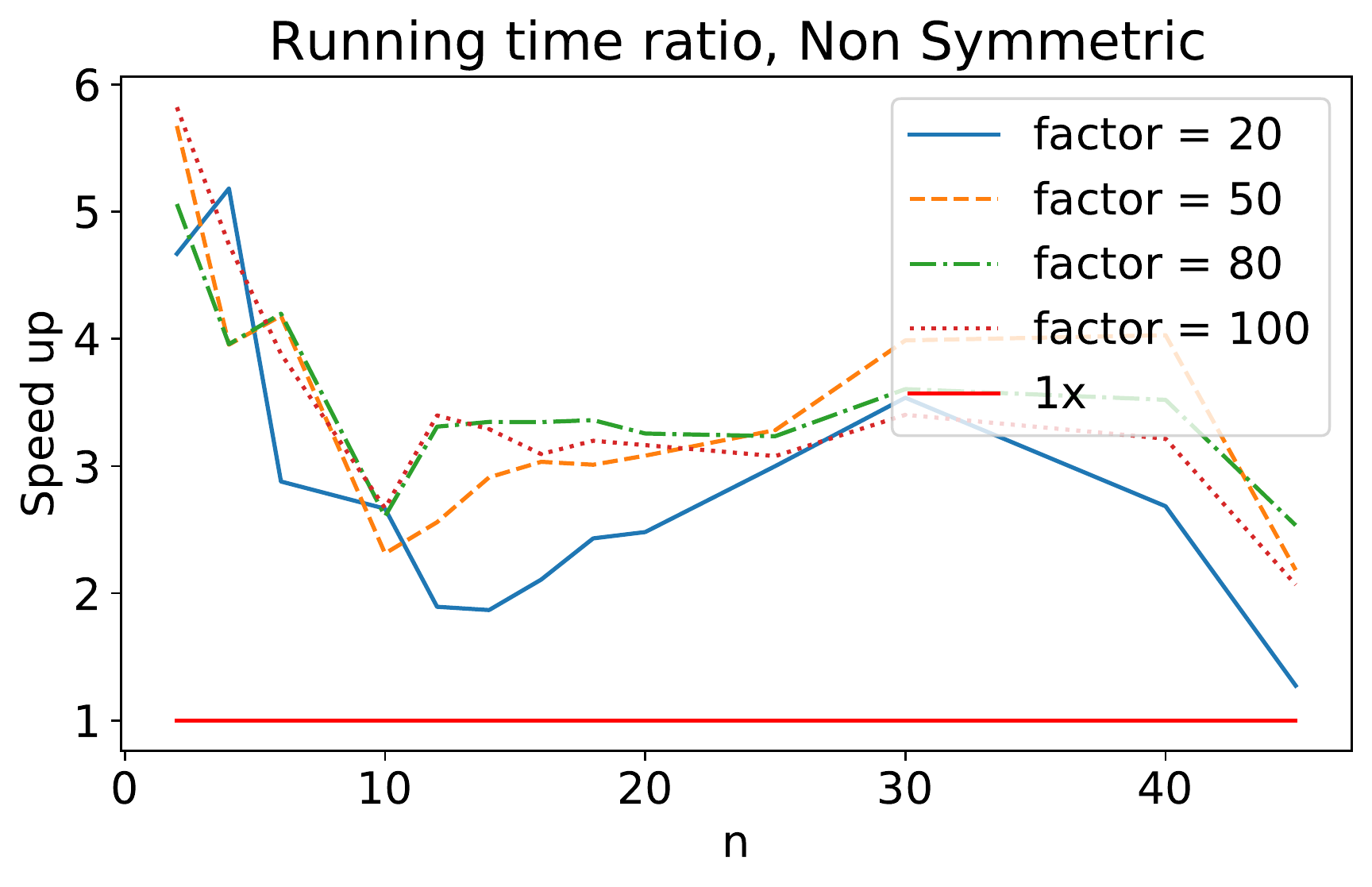}
        \caption{Bottom-right: it is shown how much the \textit{optimized-reset} algorithm is faster than \textit{det4} in case $N=(n+1)\times\text{factor}$. 
        }
    \label{fig:factor_choice}
\end{figure}
Therefore first note, that as Figure \ref{fig:factor_choice} shows, choosing any number between $20$ and $80$, in place of $50$, has similar effects if $n<70$ in the case of symmetric distribution, whilst the same holds in the case of mixture of exponentials (non symmetric) if $n\leq40$.
We think that this effect is due to the fact that ``experimentally'' there exists a long plateau in the running time of the optimized Algorithms with and without reset, see Figure \ref{fig:bVSo}.
Therefore, let us suppose that there exist $k,K$ s.t.~for any $k(n+1)\leq N^{(1)},N^{(2)}\leq K(n+1)$ and $n<40$, then $\bar C(N^{(1)},n+1)\approx \bar C(N^{(2)},n+1)$, where $\bar C(\cdot,n+1)$ is the computational cost to reduce $\cdot$ number of points in $\mathbb{R}^{n}$ using Algorithm~\ref{euclid_opt}. Moreover, we suppose that the 
$\argmin_x \bar C(x,n+1) \in [k(n+1),K(n+1)] $.
The previous two conditions are equivalent to the presence of the plateau in Figure \ref{fig:bVSo}, in correspondence with the minimum value of the running time, for the optimized Algorithms with and without reset.
Under these assumptions, the best choice would be $N^*_n=K(n+1)$.
Without knowing the value of $K$, however we can estimate the difference into
the complexity for different group subdivisions: if we have $N\gg K(n+1)$ number of points, using the ``divide and conquer'' paradigm with $N^{(i)}$ groups, we can build two algorithms s.t.~the difference of the computational costs is
\begin{align}\label{eq:difference_factor}
{O}\left(Nn+\log_{N^{(1)}/n}\right.&\left.(N/n)C(N^{(1)},n+1)\right)-{O}\left(Nn+\log_{N^{(2)}/n}(N/n)C(N^{(2)},n+1\right)\approx\\\approx&\,\,\bar{C}(Kn,n+1)\log_{N^{(2)}/n}(N/n)\left(\frac{1}{\log_{N^{(2)}/n}N^{(1)}/n}-1\right).
\end{align}
Therefore, we have that the difference depends on a factor $|1/\log_{N^{(2)}/n}(N^{(1)}/n)-1|$. 
Given Figure \ref{fig:factor_choice}, we have estimated approximately $k=20$, $K=80$ for the symmetric case, thus as a rule of thumb we assume that $N^*_n=50(n+1)$ is a reasonable value, and in view of Equation \eqref{eq:difference_factor} we can say that changing slightly $50$ the running time would remain stable.
The analogous argument can be made for the mixture of exponentials.

\section{A hybrid algorithm.}\label{sec:combine}
As mentioned in the introduction, the strategy of our randomized Algorithm~\ref{euclid_opt} is very different to the deterministic ones, and one can combine both to form a new algorithm.
The randomized Algorithm~\ref{euclid_opt} runs into trouble when the independence assumption for the cone basis in Theorem~\ref{th:main} is not met which can happen in datasets with highly correlated features; on the other hand, the deterministic algorithms have the disadvantage that they need to complete a full run over the whole dataset even when geometric greedy sampling could have finished much earlier. 
We give the details for this hybrid Algorithm \ref{algo:combined} below; it has a worst case running time of the same order as the deterministic Algorithms \cite{ Litterer2012, maria2016a,Maalouf2019} but in return has a very good chance of terminating faster.
We demonstrate this by benchmarking it against the same datasets for fast least square solvers that were used in \cite{Maalouf2019}.

\begin{algorithm}\caption{Combined measure reduction algorithm}\label{algo:combined}
\begin{algorithmic}[1]
    \Procedure{Reduce-Combined}{A set $\bx$ of $N$ points in $\R^n$,  
    $\mu=\{w_i\}$}
    \State{rem\_points $\gets N$}
    \While{rem\_points$>n+1$}
    \State{Subdivide the points $\bx$ in $G\wedge$ rem\_points groups $\{\bx_j\}_{j=1}^{G\wedge\text{rem\_points}}$}%
    \State{Compute $\bar \bw_j=\sum_{w_i\,:\,x_i\in \bx_j}w_i$, 
    $\bar \bx_j = \sum_{x_i\in g_j} w_i x_i/\bar \bw_j$
    } 
    \For{\textit{\#\_trials} times}
    \State{$\textbf{b} \gets $ $n$ random vectors from $\{\bar\bx_j-\sum_j \bar \bw_j\bx_j\}_{j=1}^G$}
    \State{$\bx^\star, w^\star\gets$ Algorithm \ref{euclid_opt} with the points $\{\bar\bx_j-\sum_j \bar \bw_j\bx_j\}$ using $\textbf{b}$ as cone basis}\label{step:random_trial}
    \If{Algorithm \ref{euclid_opt} has found a solution}
     \State{Exit for}
     \EndIf
     \EndFor
     \If{Algorithm \ref{euclid_opt} has \textbf{not} found a solution}
     \State{$\bx^\star, w^\star\gets$ Deterministic Algorithm (e.g. \cite{maria2016a,Litterer2012,Maalouf2019}) with measure $\{\bar\bx_j\}$, $\{\bar\bw_j\}$}
 	\EndIf
 	\State{$\bx\gets$ $\bx\setminus\{ x_i$ s.t. $ x_i\in \bx_j$ and $ \bar\bx_j\in \bx^\star\}$}\Comment{Eliminate the points}
 	\State{rem\_points $\gets N-$Cardinality$(\{ x_i$ s.t. $ x_i\in \bx_j$ and $ \bar\bx_j\in \bx^\star\})$}
 	\State{$\{w_i\}\gets$ $\{w_i\}\setminus\{ w_i$ s.t. $ x_i\in \bx_j$ and $ \bar\bx_j\in \bx^\star\}$}
 	\State{$\{w_i\}\gets$ $\{w_i\times w^\star_j$ s.t. $x_i\in\bar\bx_j\}$}
 	\Comment{Recalibrate the weights}
     \EndWhile
     \State{\textbf{return}  $(\bx^\star,w^\star)$ }
    \EndProcedure
  \end{algorithmic}
\end{algorithm}

\begin{figure}[hbt!]
\centering
		\includegraphics[height=2.6cm,width=0.32\textwidth]{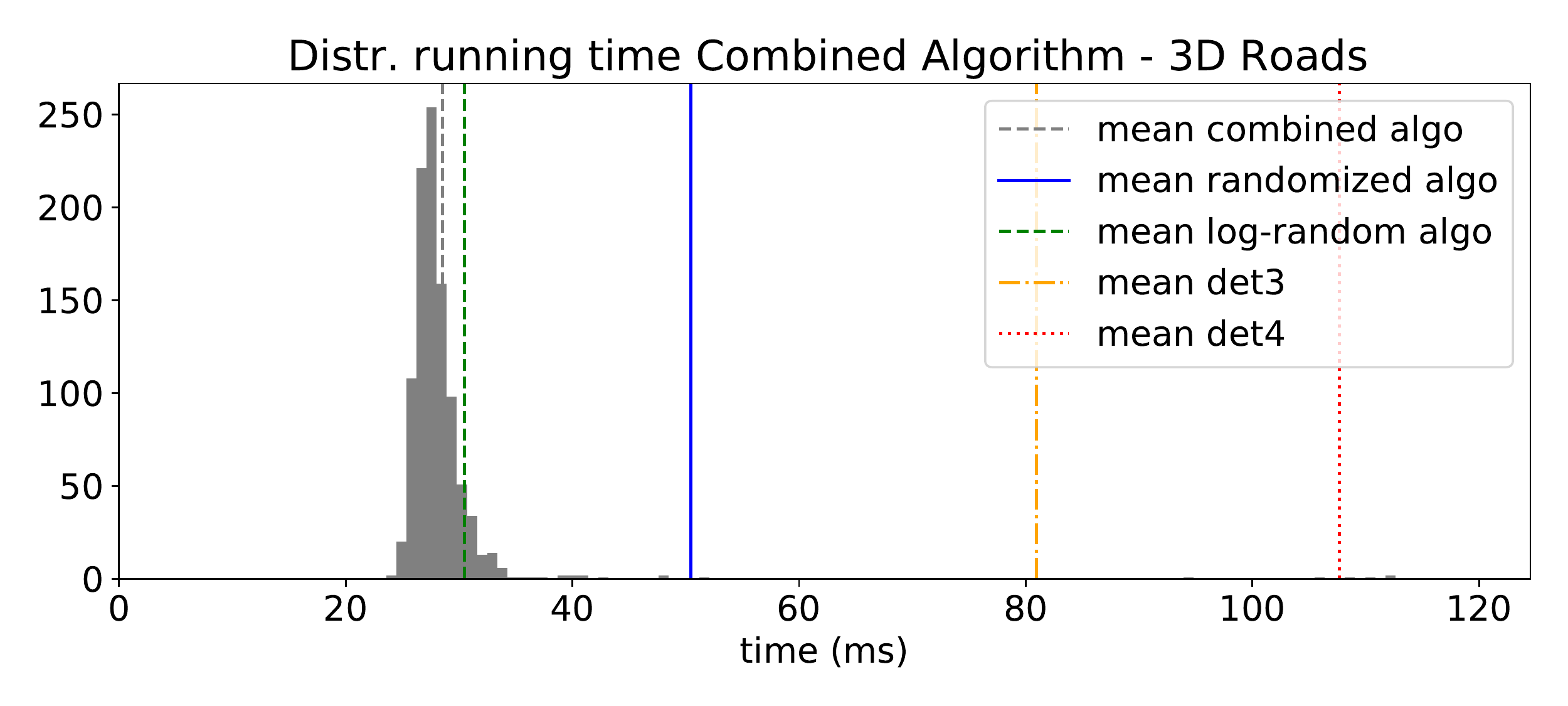}
		\includegraphics[height=2.6cm,width=0.32\textwidth]{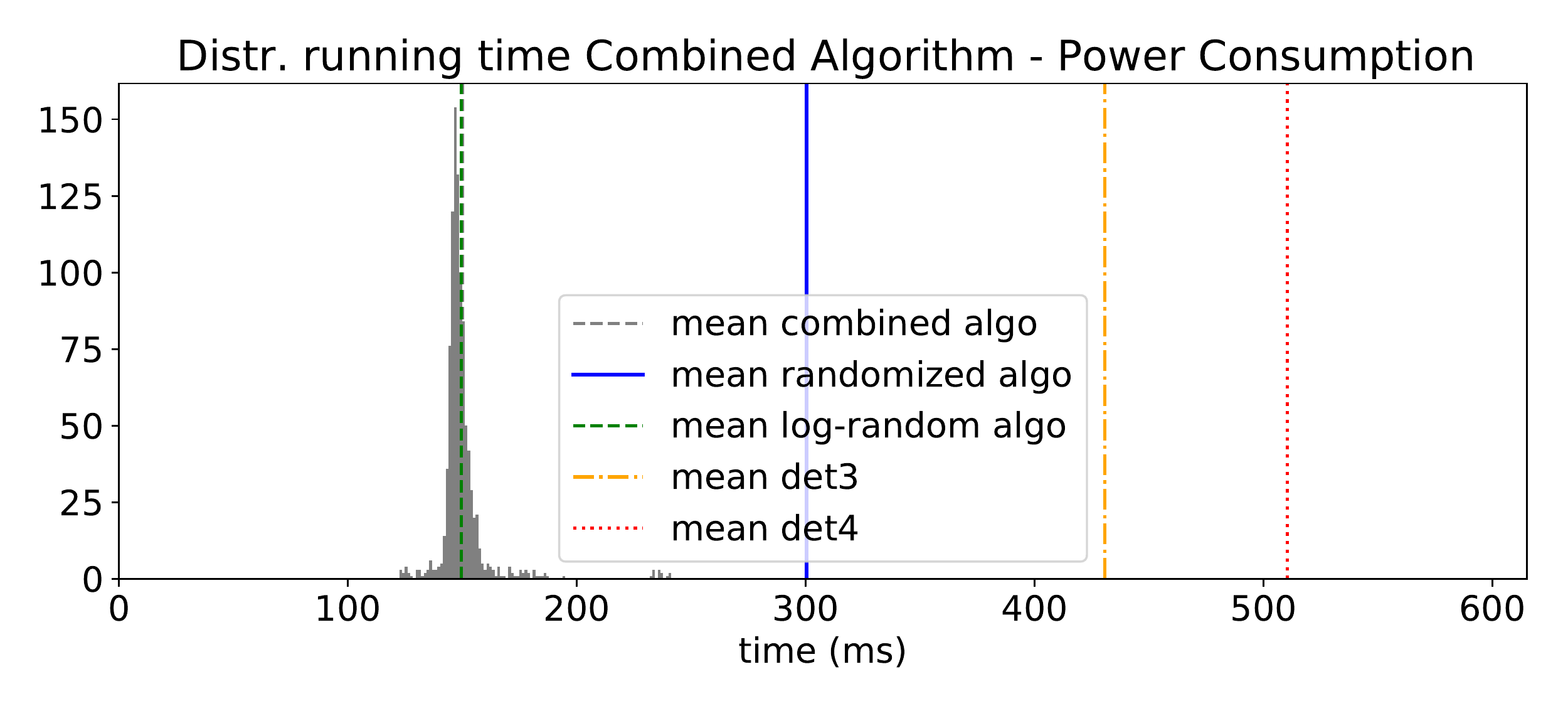}
        \includegraphics[height=2.58cm,width=0.32\textwidth]{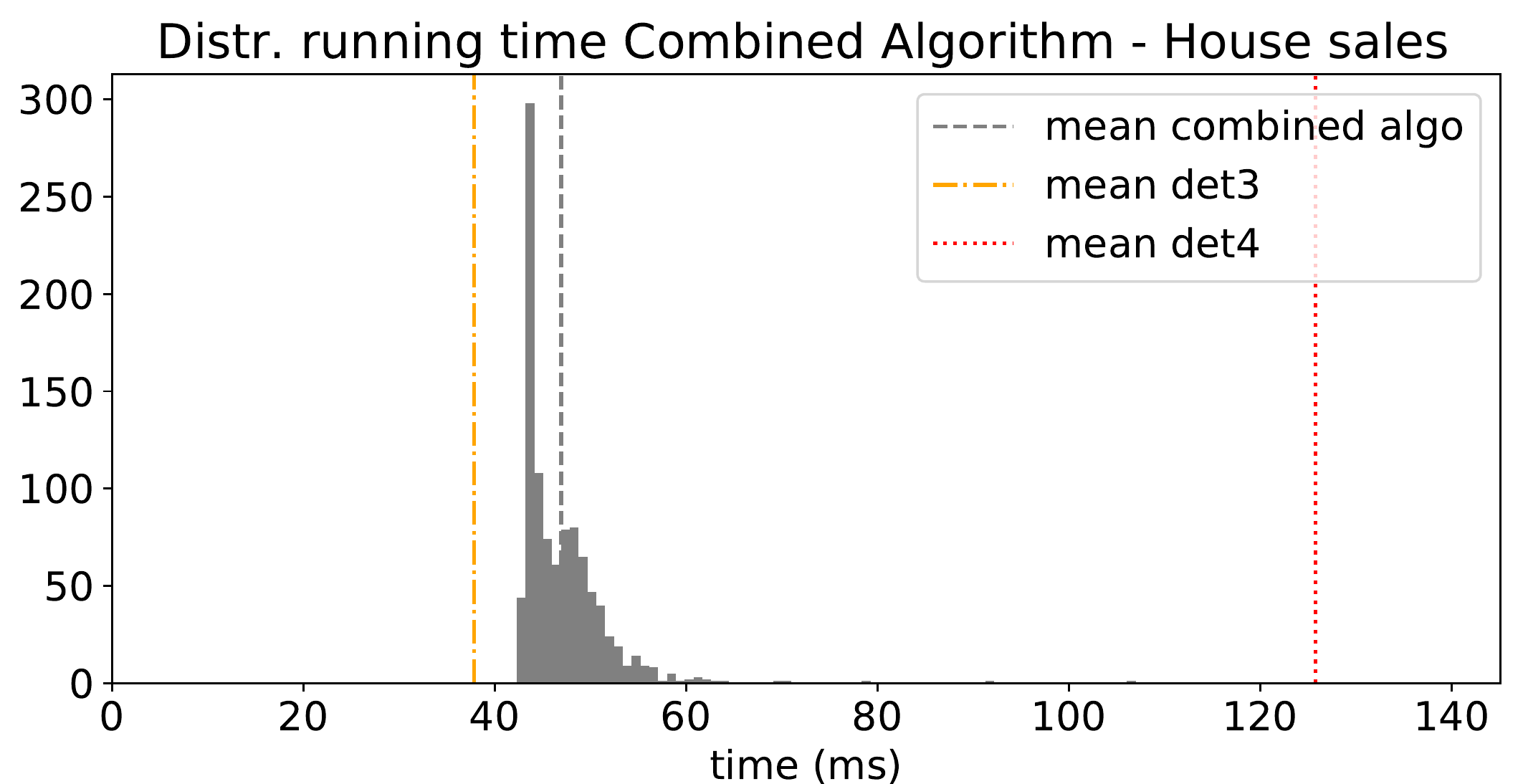}
        \caption{The running time of the Combined Algorithm \ref{algo:combined}. Note that the mean of the log-random Algorithm is essentially equal to the one of the Combined Algorithm, except for the dataset \cite{housesalesdata} where the random-algo does not work (cf. discussion above).
          For \cite{housesalesdata} the features are highly correlated, so Algorithm \ref{euclid_opt} fails, but the additional time necessary to Algorithm \ref{algo:combined} is only due to checking if the sampled basis are invertible; hence, in this case Algorithm \ref{algo:combined} behaves as it was deterministic.
          We have chosen \textit{\#\_trials}$=10$ and $G=50(n+1)$. 
        } \label{fig:combine}
\end{figure}

Some observations:
\begin{enumerate*}[label=(\roman*)]
\item Algorithm~\ref{algo:combined} always finds a solution;
\item in the case of dataset ``with linear dependence'', e.g. \cite{housesalesdata}, if PCA reductions are not allowed as in the case of the application of \cite{Maalouf2019}, the basis $\textbf{b}$ won't be invertible and therefore the complexity of the Algorithms in \cite{maria2016a,Litterer2012,Maalouf2019} would worsen of ``only'' the complexity to check \textit{\#\_trials} times if a $n\times n$ matrix is invertible, %
i.e. \textit{\#\_trials}$\times O(n^3)$, see Figure~\ref{fig:combine};
\item Algorithm~\ref{euclid_opt} must be run without reset, however note that we can add a reset strategy changing the number of iterations allowed to Algorithm~\ref{euclid_opt} at step~\ref{step:random_trial}; 
\item following the guidelines of Section~\ref{sec:choice_div&conq}, $G=50(n+1)$ for ``small $n$''.
\end{enumerate*}

To sum up, as Figure~\ref{fig:combine}\footnote{For practical reasons in our implementation of Algorithm~\ref{algo:combined} we have used our implementation of \textit{det3}, see Section~\ref{sec:implemenation}.}, the hybrid Algorithm~\ref{algo:combined} takes advantage of both the greedy geometric sampling of Algorithm \ref{euclid_opt} and robustness of the deterministic Algorithms for highly correlated datasets.
Totalled over all the datasets Algorithm~\ref{algo:combined} is the fastest.
Many variations of the above hybrid algorithm are of course possible.

\section{Implementation and benchmarking.}\label{sec:implemenation}
For \textit{det4} we have used the code provided by the authors of \cite{Maalouf2019} available at the repository\footnote{\url{https://github.com/ibramjub/Fast-and-Accurate-Least-Mean-Squares-Solvers}} in Python; for \textit{det3} the code of \cite{maria2016a} has not been written in Python, therefore we have implemented it to allow for a fair comparison using standard Numpy libraries.
We used throughout the same codeblocks for the Divide and Conquer strategy in all the implementations. %
We did not use the tree data-structure in \cite{maria2016a} since it does not not change complexity bounds and is independent of the reduction procedure itself; however, it could be also used for \textit{det4} and our randomized algorithms.  
Code for all experiments is available in public repository\footnote{\url{https://github.com/FraCose/Recombination_Random_Algos}}.

\end{document}